\documentclass{article}

\usepackage[preprint]{neurips_2026}

\usepackage[utf8]{inputenc}
\usepackage[T1]{fontenc}
\usepackage{hyperref}
\usepackage{url}
\usepackage{booktabs}
\usepackage{amsfonts}
\usepackage{amsmath}
\usepackage{amssymb}
\usepackage{amsthm}
\usepackage{microtype}
\usepackage{xcolor}
\usepackage{algorithm}
\usepackage{algorithmic}
\usepackage{graphicx}
\usepackage{multirow}
\usepackage{float}   % [H] figure placement in appendix

\newcommand{\V}{\mathbf{V}}
\newcommand{\C}{\mathbf{C}}

\newcommand{\R}{\mathbb{R}}
\newcommand{\E}{\mathbb{E}}
\newcommand{\tr}{\operatorname{tr}}
\newcommand{\EY}{\mathcal{L}_{\mathrm{EY}}}
\newcommand{\sgn}{\operatorname{sgn}}

\newtheorem{proposition}{Proposition}

\title{TreeCCA: Canonical Correlation Analysis via Gradient-Boosted Trees}

\author{}

\author{%
  James Chapman \\
  \texttt{chapmajw@gmail.com} \\
}

\begin{document}

\maketitle

\begin{abstract}
Gradient-boosted trees dominate tabular machine learning, yet canonical correlation
analysis has always relied on linear or neural encoders.  We propose \textbf{TreeCCA},
the first method to train gradient-boosted tree ensembles end-to-end as CCA encoders,
inheriting their plug-and-play reliability: no architecture design, familiar
hyperparameters, and strong performance with defaults.  The technical enabler is the
Eckart-Young (EY) loss, which supplies closed-form per-sample gradients that slot
directly into any standard GBT library (XGBoost, LightGBM) as a custom objective.

TreeCCA is the first CCA method to combine nonlinear accuracy with native
interpretability: every tree split selects one feature, so gain importances reveal which
inputs drive cross-view correlation at no extra cost.  We demonstrate these properties on synthetic benchmarks,
where TreeCCA matches or exceeds Deep CCA (2.61 vs.\ 2.43 on Signed Power; 2.93
vs.\ 2.89 on Hermite), and on a sparse benchmark with zero linear cross-view covariance,
where TreeCCA recovers the true support with $\text{Precision@}S = 1.00$ at $p=50$
while PMD finds no signal.  On the UCI HAR sensor-fusion benchmark, TreeCCA achieves
comparable accuracy to Deep CCA at $5\times$ lower cost, while XGBoost gain importances
directly validate a physics-motivated hypothesis about the data --- an interpretation
not readily available with neural encoders.  Across five popular tabular multi-view
datasets, TreeMCCA consistently matches or exceeds linear CCA in both nonlinear
correlation extraction and downstream classification accuracy.
\end{abstract}

\section{Introduction}

Canonical Correlation Analysis (CCA)~\citep{hotelling1936} finds projections of
two matched views that maximise cross-view correlations.  Multi-omics is one domain
where CCA has found significant application: integrating genomics, transcriptomics, and
proteomics has revealed gene-expression-to-genotype associations and brain--behaviour
relationships that no single-view analysis could uncover~\citep{parkhomenko2007genomic, smith2015positive}.
Gradient-boosted trees dominate tabular benchmarks~\citep{grinsztajn2022trees, shwartz2022tabular},
yet no CCA method has ever used them as encoders.

We propose \textbf{TreeCCA}: the first method to train gradient-boosted tree ensembles
as CCA encoders.  The enabling observation is that the Eckart-Young (EY) loss~\citep{chapman2023}
provides closed-form per-sample gradients and a no-spurious-local-minima guarantee,
properties that make CCA tractable via any GBT library's custom objective API with no
modification to the library itself.  Two practical implementation insights complete the
picture: an incremental embedding cache and principled unit Hessians
(\S\ref{sec:method}).

\paragraph{Contributions.}
\begin{itemize}
  \item \textbf{First GBT-based CCA.}  The EY loss provides closed-form per-sample
        gradients that plug directly into any GBT library's custom objective API with no
        modification to the library, and its global optimization landscape is
        well-characterized (\S\ref{sec:background}).
  \item \textbf{Efficient training.}  An incremental embedding cache ($50\text{--}100\times$
        speedup at $R=500$) and principled unit Hessians ensure stability and speed
        (\S\ref{sec:method}).
  \item \textbf{TreeCCA outperforms Deep CCA on synthetic benchmarks} (Signed Power:
        $2.61 \pm 0.04$ vs.\ $2.43 \pm 0.04$; Hermite: $2.93 \pm 0.02$ vs.\
        $2.89 \pm 0.02$, 5 seeds) (\S\ref{sec:sup_bench}).
  \item \textbf{Generalisation at scale.}  On Split MNIST ($N=54\text{k}$, $K=50$),
        TreeCCA outperforms Deep CCA on validation TCC (total correlation captured; 30.4
        vs.\ 25.5); DCCA overfits (train/val $= 1.95\times$), TreeCCA does not ($1.04\times$)
        (\S\ref{sec:split_mnist}).
  \item \textbf{Nonlinear feature recovery.}  On a sparse benchmark with zero-linear-covariance
        signal, TreeCCA achieves perfect precision while PMD stays at random baseline.
        On UCI HAR, the learned feature importances are consistent with a physics-motivated
        hypothesis about magnitude-driven cross-sensor signal (\S\ref{sec:sparse_recovery}).
  \item \textbf{Extensions.}  TreeMCCA generalises to $M>2$ views and is evaluated on
        five real-world multi-view benchmarks (\S\ref{sec:realworld_mat}; synthetic
        four-view validation in Appendix~\ref{app:multiview}).  A siamese single-encoder
        variant (TreeCCA-SSL) shows that tree encoders can exploit cross-view signals
        structurally invisible to any affine encoder, opening an avenue for tree-based
        tabular SSL (Appendix~\ref{app:ssl}).
\end{itemize}

\section{Related Work}
\label{sec:related}

\paragraph{Linear and kernel CCA.}
Classical CCA~\citep{hotelling1936, hardoon2004} finds optimal linear projections in
closed form; \citet{uurtio2017tutorial} and \citet{yang2019survey}
provide comprehensive reviews of CCA methods, regularisation variants, and applications.
Kernel CCA (KCCA) applies CCA in reproducing kernel Hilbert
spaces~\citep{hardoon2004}, achieving $O(N^2)$ memory complexity and requiring
regularisation for generalisation.  A unified implementation of many CCA variants is available as an open-source
library~\citep{chapman2021cca}.

\paragraph{Deep CCA.}
\citet{andrew2013} introduced Deep CCA, training neural encoders end-to-end on a CCA
objective.  \citet{wang2015stochastic} extended it to stochastic mini-batch objectives
and \citet{livne2021dvcca} to variational formulations.
The Eckart-Young (EY) loss was developed as an unconstrained CCA
objective and demonstrated with neural encoders at scales up to 582K features (UK Biobank genetics).
Deep CCA is the primary nonlinear CCA baseline, but neural networks frequently underperform
gradient-boosted trees on tabular benchmarks~\citep{grinsztajn2022trees, shwartz2022tabular}
due to threshold structure, feature interactions, and moderate sample sizes --- precisely
the regime where multi-omics and clinical CCA applications live.

\paragraph{Sparse CCA.}
PMD~\citep{witten2009pmd} is the dominant practical approach: an L1 penalty on loading
vectors identifies the sparse feature subset driving cross-view correlation.
Considerable research effort has gone into the underlying combinatorial problem ---
exact sparse CCA is NP-hard~\citep{li2024sparse} --- but no sparsity penalty overcomes
the linear objective's blindness to signals with zero linear cross-view covariance.

TreeCCA provides structural sparsity without any penalty: every split selects one
feature, so gain concentrates on the discriminative subset.  Feature diagnostics
(gain, cover, split frequency) are richer than a scalar loading, and nonlinear
cross-view signals are handled naturally.

\paragraph{Gradient-boosted trees on tabular data.}
Extensive benchmarks~\citep{grinsztajn2022trees, shwartz2022tabular} show that XGBoost
and LightGBM~\citep{chen2016, ke2017lightgbm} outperform neural networks on most tabular
datasets.  Despite this dominance, no prior work uses GBTs as CCA encoders.

Despite the superficially similar name, Canonical Correlation Forests~\citep{rainforth2015}
are entirely unrelated: they are a supervised classification method that uses CCA
\emph{inside} each tree node as a split criterion, with no connection to multi-view CCA
or unsupervised representation learning.

\paragraph{Self-supervised learning for tabular data.}
SCARF~\citep{bahri2022scarf}, SubTab~\citep{ucar2021subtab}, and
SAINT~\citep{somepalli2021saint} are neural tabular SSL methods.
Joint-embedding objectives such as VICReg~\citep{bardes2022vicreg} and Barlow
Twins~\citep{zbontar2021barlow} have driven progress in vision SSL but use neural
encoders throughout.  TreeCCA opens the possibility of joint-embedding SSL with tree
encoders, which we hypothesise may be better suited to tabular feature structure.
The primary barrier remains augmentation design: no canonical structure-preserving
augmentation exists for general tabular data, and we leave this to future work.

\section{Preliminaries}
\label{sec:background}

\subsection{The Eckart-Young (EY) Loss}

CCA seeks matrices $W_1 \in \R^{p_1 \times K}$ and $W_2 \in \R^{p_2 \times K}$
such that the columns of $X_1 W_1$ and $X_2 W_2$ are maximally correlated across views
and mutually decorrelated within each view.  Classical CCA solves this via a generalised
eigenvalue problem requiring inversion of within-view covariance matrices, which is
unstable near singularity and architecture-specific (gradients pass through the encoder).

The Eckart-Young loss~\citep{chapman2023} reformulates CCA as unconstrained minimisation
over the embedding matrices $Z_1, Z_2$ directly, bypassing the eigenvalue problem
entirely.  For centred embeddings $Z_{vc} = Z_v - \bar{Z}_v \in \R^{N \times K}$, define
the symmetrised cross-covariance and within-view covariance:
\begin{align}
  \C &= \frac{1}{N-1}\bigl(Z_{1c}^\top Z_{2c} + Z_{2c}^\top Z_{1c}\bigr), \\[4pt]
  \V &= \frac{1}{N-1}\bigl(Z_{1c}^\top Z_{1c} + Z_{2c}^\top Z_{2c}\bigr).
  \label{eq:CV}
\end{align}
The loss and its gradient per sample $i$ of view 1 are:
\begin{align}
  \EY(Z_1, Z_2) &= \underbrace{-2\,\tr(\C)}_{\text{cross-view attraction}}
                  + \underbrace{\|\V\|_F^2}_{\text{covariance penalty}},
                  \label{eq:ey} \\[4pt]
  \frac{\partial \EY}{\partial Z_{1,i}} &= \frac{4}{N-1}
  \bigl(-\,Z_{2c,i} + \V\, Z_{1c,i}\bigr),
  \label{eq:grad}
\end{align}
and symmetrically for $Z_{2,i}$.  The minimum value is $-\sum_{k=1}^K \rho_k^2$, the negative sum of squared canonical correlations, and every
local minimum is a global minimum --- so optimisation provably recovers the CCA solution.

\textbf{The enabling property for TreeCCA} is architecture-agnosticity: the gradient
in Eq.~\eqref{eq:grad} depends only on the empirical embeddings $Z_1, Z_2$, not on
how they were produced.  Any encoder --- including a GBT ensemble --- can therefore be
trained end-to-end simply by supplying these per-sample gradients as a custom objective,
with the guarantee that doing so converges to embeddings that maximise the canonical
correlations.  Why alternating regression does not achieve the same, and why the EY
formulation is necessary, is discussed in Appendix~\ref{app:alt_reg}.

\subsection{GBT Training with Custom Objectives}

XGBoost and LightGBM~\citep{chen2016, ke2017lightgbm} use histogram-based split
finding and accept custom objectives that supply explicit gradient and Hessian arrays.
Each TreeCCA round fits $2K$ scalar trees; each tree costs $O(Np)$ via the histogram
approximation, giving total complexity $O(R \cdot K \cdot N \cdot p)$ for $R$ boosting
rounds, $N$ training samples, $K$ embedding dimensions, and $p$ input features per view.

\section{TreeCCA}
\label{sec:method}

\subsection{Two-Encoder TreeCCA}

Two boosters $\mathcal{B}_1: X_1 \to Z_1 \in \R^{N \times K}$ and
$\mathcal{B}_2: X_2 \to Z_2 \in \R^{N \times K}$ are trained end-to-end.
Each embedding dimension $k$ is a separate scalar booster trained on the $k$-th column
of the normalised gradient $\tilde{G}_v$ (Eq.~\ref{eq:grad_impl}), so any GBT library
with a custom scalar objective suffices.  At each round, both gradient matrices are
computed from the current $(Z_1, Z_2)$ before either booster is updated.
Algorithm~\ref{alg:treecca} gives the full procedure.  The same design works
with LightGBM; an alternative joint design using XGBoost's multi-output API is
compared in Appendix~\ref{app:design}.

\begin{algorithm}[t]
\caption{TreeCCA}
\label{alg:treecca}
\begin{algorithmic}[1]
\REQUIRE $X_1 \in \R^{N\times p_1}$,\; $X_2 \in \R^{N\times p_2}$;\;
         embedding dim $K$;\; rounds $T$;\; learning rate $\eta$
\ENSURE Boosters $\{b_{v,k}\}$;\; embed new data as
        $\hat{Z}_v = [b_{v,1}(X_v),\ldots,b_{v,K}(X_v)]$
\STATE Initialise $2K$ scalar boosters $b_{v,k}$, base margin $\leftarrow$ $k$-th PCA direction of $X_v$
\STATE $Z_v \leftarrow \mathrm{PCA}_K(X_v)$ \hfill\COMMENT{initial embeddings, $v\in\{1,2\}$}
\FOR{$t = 1, \ldots, T$}
  \STATE $\tilde{G}_v \leftarrow \tfrac{4}{N-1}\!\bigl(-Z_{\bar{v}c} + \mathbf{V}Z_{vc}\bigr)$,
         normalised \hfill\COMMENT{EY gradient from current $(Z_1,Z_2)$, both views simultaneously (Jacobi; Eq.~\ref{eq:grad_impl})}
  \STATE $b_{v,k} \mathrel{+}= \mathrm{tree}(X_v,\;\tilde{G}_{v,k},\;H{=}1)$
         \hfill\COMMENT{fit and add one new tree per booster; $v\in\{1,2\}$, $k\in\{1,\ldots,K\}$}
  \STATE $Z_{v,k} \leftarrow Z_{v,k} + \eta\;\Delta b_{v,k}^{(t)}(X_v)$
         \hfill\COMMENT{update cached embedding with new tree only; $v\in\{1,2\}$, $k\in\{1,\ldots,K\}$}
\ENDFOR
\RETURN $\{b_{1,k}\}_{k=1}^K$,\; $\{b_{2,k}\}_{k=1}^K$
\end{algorithmic}
\end{algorithm}

\paragraph{Embedding cache.}
GBT predictions are additive: $Z^{(t+1)} = Z^{(t)} + \eta\,\mathrm{tree}_t(X)$.
Re-evaluating all $t$ trees each round costs $O(R^2NK)$ total, which dominates at
$R > 200$.  We maintain a cached $Z$ and increment it by evaluating only the new tree,
reducing cost to $O(RNK)$ --- a $50\text{--}100\times$ speedup at $R=500$.

\paragraph{Gradient and Hessian.}
GBT libraries require two arrays per round: a per-sample gradient $g_i$ and Hessian
$h_i$, used to compute optimal leaf values via
$w^* = -\sum g_i\,/\,(\sum h_i + \lambda)$ ($\lambda$: L2 leaf regularisation).

\textit{Unit Hessians.}
We set $h_i = 1$ for all samples, giving denominator $n_\ell + \lambda$ and leaf
value equal to the \emph{average} gradient over the leaf --- the standard
gradient-boosted regression update (Hessian of $\tfrac{1}{2}(y-\hat{y})^2$ is also
1).  The true EY Hessian $\frac{4}{N-1}V_{kk}$ is near-zero at PCA initialisation
($V_{kk} \approx 0$), which collapses the denominator to $\lambda \approx 1$
regardless of leaf size, giving a gradient \emph{sum} over $n_\ell \approx 90$
samples rather than an average, causing immediate divergence.

\textit{Gradient normalisation.}
The raw EY gradient bracket $(-Z_{\bar{v}c} + \mathbf{V}Z_{vc})$ varies in magnitude
across components and scales as $O(1/\sqrt{N})$, so without normalisation the
effective learning rate differs per component and per dataset size.
Normalising both views jointly to a fixed target standard deviation:
\begin{equation}
  \tilde{G}_v =
  \frac{-Z_{\bar{v}c} + \mathbf{V}Z_{vc}}
       {\max\!\bigl(\sigma(G_1),\,\sigma(G_2),\,\epsilon\bigr)}
  \cdot \sigma_{\mathrm{target}},\quad \sigma_{\mathrm{target}} = 0.1,
  \label{eq:grad_impl}
\end{equation}
makes leaf updates $O(\eta \cdot \sigma_\mathrm{target}/n_\ell)$ regardless of $N$,
decoupling training speed from dataset size and making $\eta$ uniformly interpretable.
Full derivation in Appendix~\ref{app:hessian}.

\paragraph{Initialisation.}
\label{sec:init}
Each encoder's base margin is set to the unscaled PCA projection
$Z_v^{(0)} = U_v[:, {:}K]$ (economy SVD of the centred view).
All non-zero initialisations perform well in practice; we found unscaled PCA
marginally best across five seeds (Appendix~\ref{app:init}).

\paragraph{Convergence.}
The EY loss has no spurious local minima under joint gradient descent: every critical
point of $\EY(Z_1, Z_2)$ over the full embedding matrices is a global
minimum (see \S\ref{sec:background}).  This guarantee applies to the continuous joint objective
and does not extend to the alternating GBT update, which is a two-timescale non-smooth
procedure; formal convergence guarantees for alternating boosting remain open.
Empirically, TCC plateaus within 200--500 rounds with no divergence or sub-optimal
fixed points across all benchmarks (Figure~\ref{fig:benchmarks}).

\subsection{TreeMCCA: Multi-View Extension}
\label{sec:multiview_method}

The two-view EY gradient (Eq.~\ref{eq:grad}) extends naturally to $M \geq 2$ views.
The all-pairs EY objective is $\mathcal{L} = \sum_{i<j} \EY(Z_i, Z_j)$, and its gradient
with respect to view $i$ (before normalisation) is:
\begin{equation}
  G_i = -\sum_{j \neq i} Z_{jc}
  \;+\; Z_{ic}\!\left[(M-1)\,V_{ii} + \sum_{j \neq i} V_{jj}\right].
  \label{eq:multiview_grad}
\end{equation}
At $M=2$, \textbf{TreeMCCA} reduces exactly to TreeCCA.  Results on five real-world
multi-view datasets are in \S\ref{sec:realworld_mat}; a synthetic four-view benchmark
is in Appendix~\ref{app:multiview}.

\subsection{TreeCCA-SSL: Siamese Single-Encoder Variant}
\label{sec:ssl_method}

TreeCCA extends to self-supervised learning via a siamese variant in which a single
shared booster is trained on independently augmented views.  The algorithm and
theoretical analysis are deferred to Appendix~\ref{app:ssl}.

\section{Experiments}
\label{sec:experiments}

Default hyperparameters and hardware are listed in Appendix~\ref{app:details}.
Pseudocode for all new methods is provided in Appendix~\ref{app:ssl} and~\ref{app:multiview}; code will be made available via GitHub upon acceptance.

\subsection{TreeCCA as a Nonlinear CCA: Comparison with Kernel CCA and Deep CCA}
\label{sec:sup_bench}

\paragraph{Synthetic benchmarks.}
We evaluate on two synthetic benchmarks with $K=3$ shared latent dimensions, 5 Gaussian
noise dimensions per view ($p_v = 8$), $N=3000$, 20\% test split, noise $\sigma=0.15$.
Performance is measured by \textbf{Total Correlation Captured} (TCC)
$= \sum_{k=1}^K |\rho_k^{\mathrm{test}}|$, maximum $K$.
The \textbf{Signed-Power} benchmark encodes latents $z_k \sim \mathcal{N}(0,1)$ as a
cube root in one view and a cube in the other --- bijective but nonlinear maps that
reduce the linear feature-to-latent correlation to $\approx 0.55$.
The \textbf{Hermite} benchmark uses $x_{1k} = H_2(z_k) = z_k^2 - 1$ (even) and
$x_{2k} = H_3(z_k) = z_k^3 - 3z_k$ (odd): probabilists' Hermite polynomials are
orthogonal under the Gaussian measure, so $\mathrm{corr}(H_2(z), H_3(z)) = 0$ for
$z \sim \mathcal{N}(0,1)$.  Linear CCA is therefore structurally blind to the signal ---
the views share no linear cross-covariance regardless of how the projections are chosen.

Table~\ref{tab:nonlinear} reports peak test TCC over 5 seeds ($\{42,0,1,2,3\}$,
500 rounds).  TreeCCA tops both baselines: on Signed Power it reaches $2.61 \pm 0.04$
vs.\ $2.43 \pm 0.04$ for Deep CCA ($+7.5\%$) and $2.28 \pm 0.03$ for KCCA; on
Hermite all three nonlinear methods far exceed linear CCA (near-zero at $0.14$), with
TreeCCA at $2.93 \pm 0.02$ edging Deep CCA ($2.89$) and KCCA ($2.68$).
Convergence curves are shown in Appendix~\ref{app:convergence} (Figure~\ref{fig:benchmarks}); TreeCCA reaches competitive TCC
within 100--150 rounds, without the $O(N^2)$ memory of KCCA.  Wall-clock scaling with $N$ is characterised in
Appendix~\ref{app:scalability}; multi-seed stability curves are in
Appendix~\ref{app:multiseed}.

\begin{table}[h]
  \centering
  \caption{Peak test TCC (mean $\pm$ std, 5 seeds, $N=3000$, $K=3$, 500 rounds).
           KCCA: RBF kernel, $c=0.1$, $\gamma=0.1$, $O(N^2)$ training, $\sim$13s.
           DCCA: EY loss, 128--64 hidden units, Adam, 1000 epochs.
           TreeCCA: 500 rounds, $\sim$50s.}
  \label{tab:nonlinear}
  \begin{tabular}{lcccc}
    \toprule
    Benchmark & Linear CCA & KCCA (RBF) & DeepCCA & TreeCCA (ours) \\
    \midrule
    Signed Power & $1.63 \pm 0.04$ & $2.28 \pm 0.03$ & $2.43 \pm 0.04$ & $\mathbf{2.61 \pm 0.04}$ \\
    Hermite      & $0.14 \pm 0.06$ & $2.68 \pm 0.05$ & $2.89 \pm 0.02$ & $\mathbf{2.93 \pm 0.02}$ \\
    \midrule
    Training time      & $<$0.1s & $\sim$13s  & $\sim$300s & $\sim$50s \\
    Memory             & $O(Np)$ & $O(N^2)$   & $O(Np)$        & $O(Np)$       \\
    Scalable $N>20$k   & \checkmark & $\times$ & \checkmark  & \checkmark    \\
    Feature importance & $\times$   & $\times$ & $\times$    & \checkmark    \\
    \bottomrule
  \end{tabular}
\end{table}

\subsubsection{Scale: Split MNIST}
\label{sec:split_mnist}

To test whether the performance gap holds at realistic scale, we split the
MNIST~\citep{lecun1998mnist} digit image dataset
($N=60{,}000$, 28$\times$28 pixels each) into left and right halves (392 pixels per view),
adding 10\% dropout noise to each view independently, and train with $K=50$ components.  This is a pure
generalisation test: more components than classes (10), large $N$, and a high-dimensional
signal that a poorly regularised method will overfit.

Table~\ref{tab:split_mnist} shows the result.  Deep CCA memorises the training set
(train TCC $= 49.7$) but generalises poorly (val TCC $= 25.5$), a $1.95\times$ train/val
ratio.  TreeCCA reaches a lower training TCC ($31.6$) but generalises better ($30.4$),
a $1.04\times$ ratio.  XGBoost's built-in regularisation --- depth limits, minimum child
weight, subsampling --- prevents the encoder from fitting noise in the training split.
This is the natural behaviour of gradient-boosted trees on moderate-$N$ datasets and is
precisely the regime of most multi-omics and clinical applications.

\begin{table}[h]
  \centering
  \caption{Split MNIST: left vs.\ right half-image views ($K=50$, $N_{\mathrm{tr}}=54{,}000$,
           $N_{\mathrm{val}}=6{,}000$, seed 42).  TCC $=\sum_{k=1}^K |\rho_k|$, max $=50$.}
  \label{tab:split_mnist}
  \begin{tabular}{lccc}
    \toprule
    Method & Train TCC & Val TCC & Train/Val \\
    \midrule
    Linear CCA    & 13.87 & 9.41  & $1.47\times$ \\
    DeepCCA (EY)  & 49.69 & 25.46 & $1.95\times$ \\
    TreeCCA (ours) & 31.59 & \textbf{30.44} & $1.04\times$ \\
    \bottomrule
  \end{tabular}
\end{table}

\subsubsection{Sensor Fusion: UCI HAR}
\label{sec:realworld}

The UCI HAR dataset~\citep{anguita2013har, dua2019uci} records inertial sensor data from smartphones
worn by 30 volunteers performing six activities.  We use the raw 2.56-second windows and
independently feature-extract the accelerometer and gyroscope channels, treating them as
two views.  Each view is represented by 36 features: 9 time-domain statistics (mean, std,
min, max, q25, q75, RMS, IQR, range) per axis $\times$ 3 axes, plus 9 magnitude features
(the L2 norm of each statistic across axes).  The magnitude features are motivated by the
Newton--Euler rotational dynamics: the centripetal acceleration $\omega \times (\omega \times r)$
is quadratic in the angular velocity magnitude $|\omega|$, a signal that is nonlinear in any
single axis but present in the norm.

We train with $K=6$ (one component per activity class) using 5 seeds.  Linear CCA and
TreeCCA are evaluated on test TCC and linear probe accuracy for the 6-way activity
classification task.

\begin{table}[h]
  \centering
  \caption{UCI HAR accelerometer $\leftrightarrow$ gyroscope ($K=6$, $p=36$, 5 seeds).
           Linear CCA is deterministic (no std). Deep CCA: EY loss, 128--64 hidden, Adam.}
  \label{tab:har}
  \begin{tabular}{lccc}
    \toprule
    Method & Test TCC & Probe acc. & Time/seed \\
    \midrule
    Linear CCA             & $2.63$                    & $0.182$                    & $<$1s \\
    Deep CCA               & $4.67 \pm 0.06$           & $0.510 \pm 0.010$ & 80s \\
    TreeCCA (ours)         & $4.32 \pm 0.22$           & $0.466 \pm 0.021$ & 17s \\
    \bottomrule
  \end{tabular}
\end{table}

Both TreeCCA and Deep CCA substantially outperform linear CCA ($+64\%$ TCC).
Deep CCA has a small edge in TCC and probe accuracy, at $5\times$ the wall-clock cost.
TreeCCA does not require a GPU; for very large-scale applications, GPU-accelerated
training is a drop-in option via XGBoost and LightGBM's native GPU backends.
The deeper practical advantage is ease of iteration: gradient-boosted trees require no
architecture search, perform well at the moderate sample sizes typical of multi-omics
and clinical CCA applications, and are robust to hyperparameter choice with sensible
defaults.  Neural encoders, by contrast, typically require tuning of architecture width,
depth, and learning rate to realise their potential.  A systematic sweep over depth and
learning rate confirms TreeCCA's robustness (Appendix~\ref{app:hp}).  For a
practitioner starting a new multi-view analysis, TreeCCA is the natural first step.  The further distinction is interpretability:
unlike neural encoders, XGBoost's gain importances identify which features drive the
cross-sensor signal, making the model's reasoning legible to practitioners.  We examine
those importances directly in \S\ref{sec:har_importance}, where they are consistent with
a physics-motivated hypothesis about angular velocity magnitude --- an analysis
structurally impossible to perform with a neural encoder.

\subsection{Nonlinear Feature Recovery: TreeCCA vs.\ Sparse CCA}
\label{sec:sparse_recovery}

\paragraph{Structural impossibility of linear sparse CCA.}
PMD~\citep{witten2009pmd} assumes cross-view correlation is driven by a small subset of
features.  This is operationalised by an L1 penalty on linear loading vectors.  But the
penalty acts on the linear objective: even perfect feature selection cannot recover a
signal whose linear cross-view covariance is zero.  TreeCCA achieves \emph{structural
sparsity} without any penalty: because each split in a decision tree selects exactly one
feature as the split criterion, gain concentrates on features that genuinely reduce the
EY loss.  Features irrelevant to the cross-view signal receive no splits and accumulate
zero gain regardless of ambient dimension.

\paragraph{Synthetic sparse recovery benchmark.}
$S=5$ true features per view encode a nonlinear signal: $x_{1k} = \sgn(z_k)|z_k|^{0.5}$
and $x_{2k} = z_k^2 - 1$ (Hermite $H_2$, zero linear cross-view covariance).  The
remaining $p - S$ features are pure Gaussian noise.  $N=500$, $K=5$, 5 seeds.
We evaluate \textbf{Precision@$S$}: fraction of the top-$S$ importance-ranked features
that are truly causal.

\begin{figure}[H]
  \centering
  \includegraphics[width=0.9\linewidth]{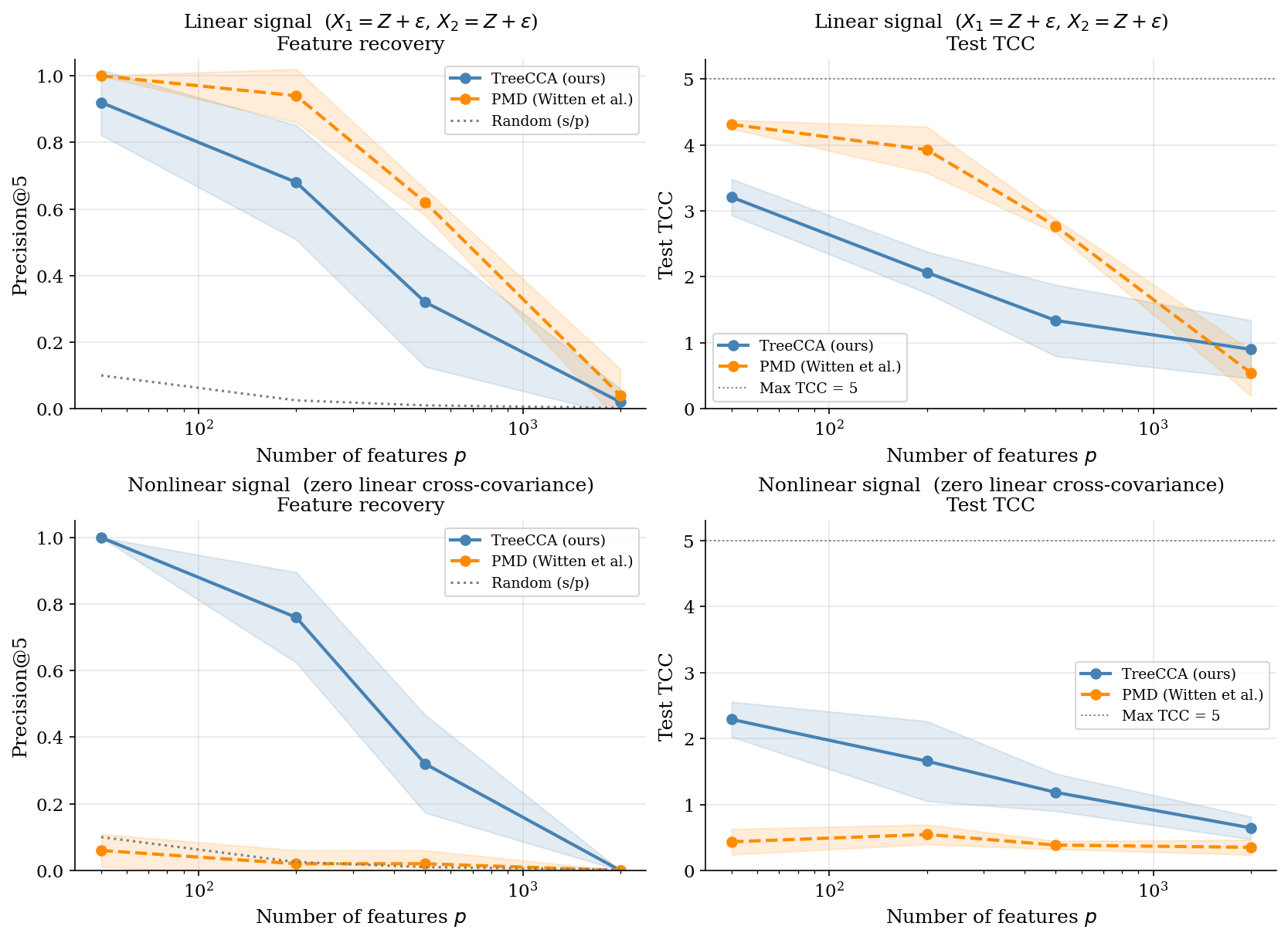}
  \caption{Sparse feature recovery on nonlinear signal ($\sgn(z)|z|^{0.5}
           \leftrightarrow z^2-1$, zero linear cross-covariance).
           \textbf{Left}: Precision@$S$ vs.\ ambient dimension $p$; TreeCCA (blue) achieves
           $1.00$ at $p=50$ and degrades gracefully, PMD (orange) stays at random baseline
           throughout.
           \textbf{Right}: TCC vs.\ $p$; TreeCCA maintains substantial TCC while PMD
           remains near zero at all $p$.
           Mean $\pm$ std over 5 seeds.}
  \label{fig:sparse_recovery}
\end{figure}

\begin{table}[h]
  \centering
  \caption{Sparse feature recovery on nonlinear signal ($S=5$, $N=500$, 5 seeds, mean $\pm$ std).}
  \label{tab:sparse_recovery}
  \begin{tabular}{llcc}
    \toprule
    $p$ & Method & Precision@$S$ & TCC \\
    \midrule
    \multirow{2}{*}{50}   & PMD     & $0.06 \pm 0.05$          & $0.43 \pm 0.19$ \\
                          & TreeCCA & $\mathbf{1.00 \pm 0.00}$ & $\mathbf{2.31 \pm 0.24}$ \\
    \midrule
    \multirow{2}{*}{200}  & PMD     & $0.02 \pm 0.04$          & $0.54 \pm 0.15$ \\
                          & TreeCCA & $\mathbf{0.78 \pm 0.13}$ & $\mathbf{1.75 \pm 0.55}$ \\
    \midrule
    \multirow{2}{*}{500}  & PMD     & $0.02 \pm 0.04$          & $0.38 \pm 0.06$ \\
                          & TreeCCA & $\mathbf{0.26 \pm 0.14}$ & $\mathbf{1.24 \pm 0.32}$ \\
    \midrule
    \multirow{2}{*}{2000} & PMD     & $0.00 \pm 0.00$          & $0.35 \pm 0.11$ \\
                          & TreeCCA & $\mathbf{0.06 \pm 0.12}$ & $\mathbf{0.60 \pm 0.17}$ \\
    \bottomrule
  \end{tabular}
\end{table}

PMD's failure here is structural, not a matter of tuning: the benchmark signal has
zero linear cross-view covariance by construction (Hermite orthogonality), so no L1
penalty on a linear objective can recover the causal features.  PMD serves as a
reference demonstrating this structural limitation, not as a competitive nonlinear
sparse baseline.  TreeCCA, by contrast, achieves perfect precision at $p=50$ and
degrades gracefully --- precision $0.78$ at $p=200$, $0.26$ at $p=500$ --- while PMD
stays near zero throughout.  Even at $p=2000$ ($p/N=4$), where neither method recovers
the true support, TreeCCA retains a substantial TCC advantage ($0.60$ vs.\ $0.35$).

\subsubsection{Interpretable Feature Selection on Real Data: UCI HAR}
\label{sec:har_importance}

If the magnitude features in the HAR dataset encode the centripetal acceleration
structure we hypothesised, they should appear prominently in XGBoost's learned feature
importances --- and should do so more strongly for the gyroscope (which directly measures
angular velocity $\omega$) than for the accelerometer (which measures the combined linear
and centripetal acceleration).

Figure~\ref{fig:har_importance} shows the result.  We use XGBoost's \emph{gain} importance:
the total reduction in the EY loss objective achieved by splits on each feature, summed
across all trees and all $K=6$ boosters.  For the gyroscope view, magnitude features
account for \textbf{47\%} of total gain, with the top two features being
\texttt{gyr\_mag\_iqr} and \texttt{gyr\_mag\_std} --- both encoding the variability of
$|\omega|$ over the window.  For the accelerometer,
magnitude features account for 26\% of gain, with single-axis range features
dominating (locomotion direction is already highly discriminative for the accelerometer).

This result makes the feature selection legible in a way that is impossible with DCCA:
the gain importances are consistent with the physical intuition that drove the feature
engineering.  Practitioners can inspect which statistics the model actually uses, remove
redundant features, and re-run in seconds.

\begin{figure}[H]
  \centering
  \includegraphics[width=\linewidth]{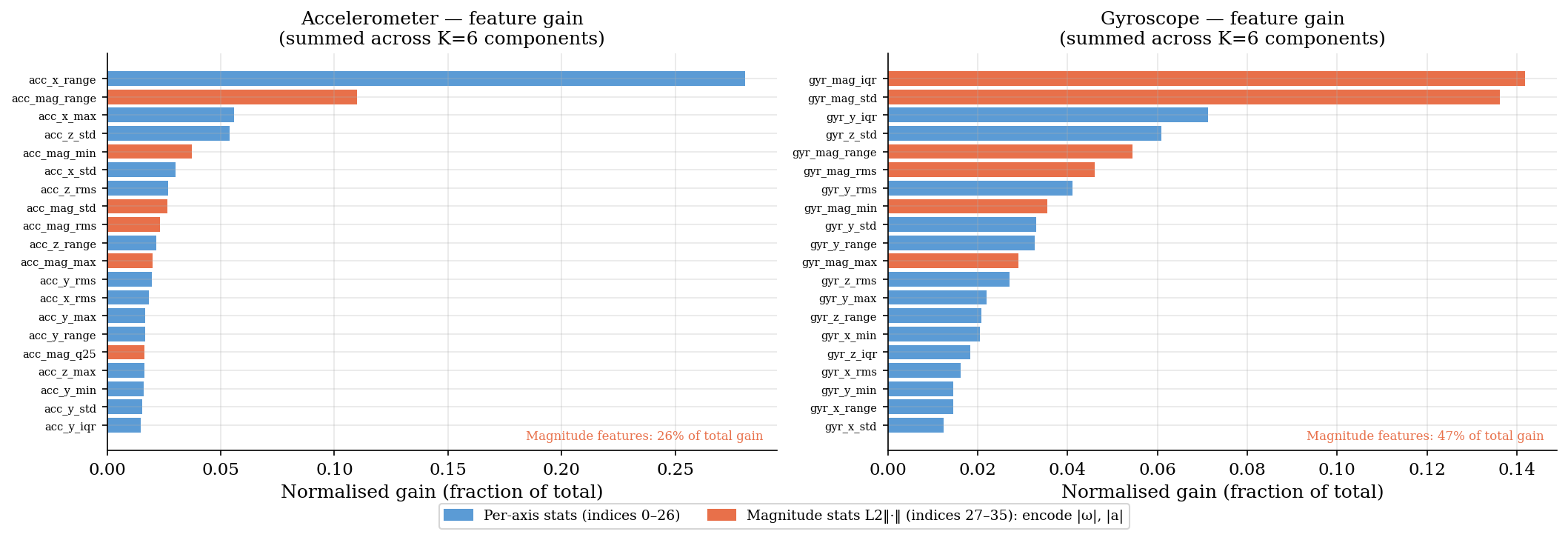}
  \caption{Top-20 XGBoost feature importances (gain, summed across $K=6$ boosters) for
           gyroscope (left) and accelerometer (right) views on UCI HAR, seed 42.
           Magnitude features (orange) account for 47\% of gyroscope gain and 26\% of
           accelerometer gain, consistent with the Newton--Euler physics motivation.
           DCCA has no equivalent: its feature weights are a $64 \times 36$ matrix
           with no interpretable structure.}
  \label{fig:har_importance}
\end{figure}

\subsubsection{Multi-View Dataset Sweep}
\label{sec:realworld_mat}

To assess breadth, we evaluate on five publicly available multi-view datasets
(Table~\ref{tab:realworld_mat}): Caltech101-7 (6 views, 7 classes),
3Sources (3 views, news articles from three outlets, 6 classes),
NUS-WIDE (5 views, image tags and features, 12 classes),
Handwritten (6 views, digits in multiple descriptor spaces, 10 classes), and
MSRC-v5 (5 views, image descriptors, 7 classes)~\citep{zhang2024multiview}.
These benchmarks have been used by recent deep CCA
methods~\citep{he2024nrdcca, jin2015crossmodal, yuan2019multiview}.
Each dataset is evaluated with $K = C{-}1$ components ($C$ classes) and 5 seeds.
Both methods use all $M$ views jointly: MCCA and TreeMCCA.
We report test TCC and linear probe accuracy on the concatenation of all $M$ view embeddings.
No preprocessing beyond standardisation is applied to TreeMCCA; XGBoost's histogram
split finding handles high-dimensional views natively.

\begin{table}[h]
  \centering
  \small
  \caption{Multi-view dataset sweep (mean $\pm$ std, 5 seeds).  $K = C{-}1$ components
           per dataset ($C$ = number of classes).
           Both methods use all $M$ views jointly: MCCA and TreeMCCA (\texttt{train\_treecca\_multiview}).
           Acc = linear probe accuracy on concatenated $M$-view embeddings.}
  \label{tab:realworld_mat}
  \begin{tabular}{l@{\;}c@{\;}c@{\quad}c@{\quad}c@{\quad}c}
    \toprule
    & & \multicolumn{2}{c}{MCCA} & \multicolumn{2}{c}{TreeMCCA} \\
    \cmidrule(lr){3-4}\cmidrule(lr){5-6}
    Dataset & $M$/$K$ & TCC & Acc & TCC & Acc \\
    \midrule
    Caltech101-7 & 6/6  & $3.77 \pm 0.14$ & $0.900 \pm 0.017$ & $\mathbf{5.16 \pm 0.09}$ & $\mathbf{0.913 \pm 0.018}$ \\
    3Sources     & 3/5  & $1.30 \pm 0.69$ & $0.176 \pm 0.083$ & $\mathbf{3.63 \pm 0.11}$ & $\mathbf{0.706 \pm 0.026}$ \\
    NUS-WIDE     & 5/11 & $3.03 \pm 0.11$ & $0.349 \pm 0.019$ & $\mathbf{5.28 \pm 0.18}$ & $\mathbf{0.365 \pm 0.022}$ \\
    Handwritten  & 6/9  & $4.56 \pm 0.03$ & $\mathbf{0.911 \pm 0.009}$ & $\mathbf{5.43 \pm 0.05}$ & $0.872 \pm 0.021$ \\
    MSRC-v5      & 5/6  & $1.67 \pm 0.07$ & $0.662 \pm 0.063$ & $\mathbf{3.47 \pm 0.22}$ & $\mathbf{0.862 \pm 0.063}$ \\
    \bottomrule
  \end{tabular}
\end{table}

TreeMCCA achieves higher TCC than MCCA on all five datasets.  Accuracy gains are
largest on 3Sources and MSRC-v5, where the cross-view signal is most nonlinear.
Handwritten is the one exception on accuracy: MCCA achieves $0.911$ vs TreeMCCA's
$0.872$, despite TreeMCCA capturing substantially more cross-view variance
(TCC $5.43$ vs $4.56$) --- higher TCC does not always imply higher linear-probe
accuracy when MCCA already embeds the discriminative signal well.  t-SNE
visualisations of the test embeddings are in Appendix~\ref{app:tsne}.

\subsection{TreeCCA-SSL}
\label{sec:ssl_exp}

Full details and results are in Appendix~\ref{app:ssl}.
On a synthetic benchmark where the cross-view signal is invisible to any affine encoder,
TreeCCA-SSL reaches accuracy $0.630$ --- $3.8\times$ above random chance ($0.167$) and
well above PCA ($0.159$) --- showing that tree encoders can exploit cross-view structure
that linear methods structurally cannot.  Principled augmentation design for general
tabular data remains an open problem.

\section{Conclusion}

We have shown that gradient-boosted trees can be trained end-to-end as CCA encoders by
treating the EY loss as a custom GBT objective (\S\ref{sec:method}).  The result is a
nonlinear CCA method with no bespoke training code and no architecture search.

The experiments demonstrate that TreeCCA offers three properties in a single method:
\begin{itemize}
  \item \textbf{Nonlinear and sparse signal recovery.}  TreeCCA matches or exceeds
        Deep CCA on synthetic benchmarks (Table~\ref{tab:nonlinear}) and generalises
        substantially better at scale ($N=54\text{k}$, $K=50$; train/val ratio
        $1.04\times$ vs.\ $1.95\times$; Section~\ref{sec:split_mnist}).  On a sparse
        benchmark with zero linear cross-view covariance, TreeCCA achieves
        $\text{Precision@}S=1.00$ at $p=50$ while PMD finds no signal at any $p$
        (Table~\ref{tab:sparse_recovery}; Section~\ref{sec:sparse_recovery}).
  \item \textbf{Interpretability.}  On UCI HAR, XGBoost gain importances directly
        validate a Newton--Euler physics hypothesis about angular velocity magnitude
        (Section~\ref{sec:har_importance}) --- an insight not readily available with
        neural encoders.
  \item \textbf{Practical breadth.}  TreeCCA achieves comparable accuracy to Deep CCA
        on UCI HAR at $5\times$ lower wall-clock cost.  TreeMCCA exceeds
        linear CCA on TCC across all five heterogeneous multi-view benchmarks, with the
        largest gains where features are most nonlinear (Table~\ref{tab:realworld_mat}).
        The method extends naturally to $M>2$ views and self-supervised learning
        (Sections~\ref{sec:multiview_method},~\ref{sec:ssl_exp}).
\end{itemize}

\paragraph{Limitations.}
Sparse recovery precision degrades at high $p/N$; column subsampling is a natural
remedy.  Formal convergence theory for alternating GBT updates remains open.  For SSL,
the primary open problem is principled tabular augmentation design.

\bibliographystyle{plainnat}
\bibliography{treecca}

\section*{NeurIPS Paper Checklist}

\begin{enumerate}

\item {\bf Claims}
    \item[] Question: Do the main claims made in the abstract and introduction accurately reflect the paper's contributions and scope?
    \item[] Answer: \answerYes{}
    \item[] Justification: All quantitative claims (TCC values, percentage improvements) are backed by reported experimental results with stated seeds, datasets, and hyperparameters.

\item {\bf Limitations}
    \item[] Question: Does the paper discuss the limitations of the work performed by the authors?
    \item[] Answer: \answerYes{}
    \item[] Justification: A Limitations paragraph in the Conclusion discusses degradation at high $p/N$, the absence of formal convergence guarantees, and the open problem of tabular augmentation design for SSL.

\item {\bf Theory assumptions and proofs}
    \item[] Question: For each theoretical result, does the paper provide the full set of assumptions and a complete (and correct) proof?
    \item[] Answer: \answerYes{}
    \item[] Justification: Proposition~\ref{prop:linear_fail} states assumptions explicitly and is proved in full in Appendix~\ref{app:ssl}.

\item {\bf Experimental result reproducibility}
    \item[] Question: Does the paper fully disclose all the information needed to reproduce the main experimental results of the paper?
    \item[] Answer: \answerYes{}
    \item[] Justification: Seeds, hyperparameters, and hardware are stated in Appendix~\ref{app:details}. All synthetic DGPs are given in closed form. Pseudocode for all new methods is provided in the appendix; code will be made available via GitHub upon acceptance.

\item {\bf Open access to data and code}
    \item[] Question: Does the paper provide open access to the data and code?
    \item[] Answer: \answerYes{}
    \item[] Justification: Pseudocode for all new methods is provided in the appendix; code will be made available via GitHub upon acceptance. All real-world datasets are publicly available (UCI HAR~\citep{anguita2013har}).

\item {\bf Experimental setting/details}
    \item[] Question: Does the paper specify all training and test details?
    \item[] Answer: \answerYes{}
    \item[] Justification: Hyperparameters, data splits, and evaluation protocols are in Appendix~\ref{app:details} and within each section.

\item {\bf Experiment statistical significance}
    \item[] Question: Does the paper report error bars?
    \item[] Answer: \answerYes{}
    \item[] Justification: Multi-seed experiments (5 seeds) report mean $\pm$ std in Tables~\ref{tab:nonlinear}, \ref{tab:har}, \ref{tab:sparse_recovery}, and~\ref{tab:realworld_mat}.
    The Split MNIST experiment (Table~\ref{tab:split_mnist}) does not report error bars;
    DCCA at this scale ($N=54\text{k}$, $K=50$, 1000 epochs) requires substantial compute per run.

\item {\bf Experiments compute resources}
    \item[] Question: Does the paper provide sufficient information on compute?
    \item[] Answer: \answerYes{}
    \item[] Justification: All experiments run on an Apple M3 MacBook Pro. Runtimes are stated throughout.

\item {\bf Code of ethics}
    \item[] Question: Does the research conform with the NeurIPS Code of Ethics?
    \item[] Answer: \answerYes{}
    \item[] Justification: Foundational statistical method with no direct path to harmful applications.

\item {\bf Broader impacts}
    \item[] Question: Does the paper discuss societal impacts?
    \item[] Answer: \answerYes{}
    \item[] Justification: The primary application domain (multi-omics) has positive scientific impact. No direct negative societal impacts beyond general dual-use concerns.

\item {\bf Safeguards}
    \item[] Question: Does the paper describe safeguards for responsible release?
    \item[] Answer: \answerNA{}
    \item[] Justification: Statistical learning algorithm and experiment code; no misuse risks analogous to generative models or scraped datasets.

\item {\bf Licenses for existing assets}
    \item[] Question: Are creators of assets properly credited?
    \item[] Answer: \answerYes{}
    \item[] Justification: All datasets cited. XGBoost and LightGBM cited~\citep{chen2016, ke2017lightgbm} and used under Apache 2.0.

\item {\bf New assets}
    \item[] Question: Are new assets well documented?
    \item[] Answer: \answerYes{}
    \item[] Justification: Pseudocode for all new methods is provided in the appendix; a public repository will be made available upon acceptance.

\item {\bf Crowdsourcing and research with human subjects}
    \item[] Question: Does the paper include crowdsourcing or human subjects details?
    \item[] Answer: \answerNA{}
    \item[] Justification: No crowdsourcing or human subjects.

\item {\bf Declaration of LLM usage}
    \item[] Question: Does the paper describe LLM usage?
    \item[] Answer: \answerNA{}
    \item[] Justification: LLMs were used for writing assistance only, not as part of the core methodology.

\end{enumerate}

% ─────────────────────────────────────────────────────────────────────────────
% Appendix (part of the paper; included via separate file for cleanliness)
% ─────────────────────────────────────────────────────────────────────────────
% ─────────────────────────────────────────────────────────────────────────────
% appendix.tex  —  TreeCCA supplementary material
% Included via \input{appendix} in treecca.tex (preprint) or compiled
% standalone via treecca_supp.tex (NeurIPS submission).
% ─────────────────────────────────────────────────────────────────────────────

\appendix
% ─────────────────────────────────────────────────────────────────────────────

\section{Convergence Curves}
\label{app:convergence}

\begin{figure}[H]
  \centering
  \includegraphics[width=\linewidth]{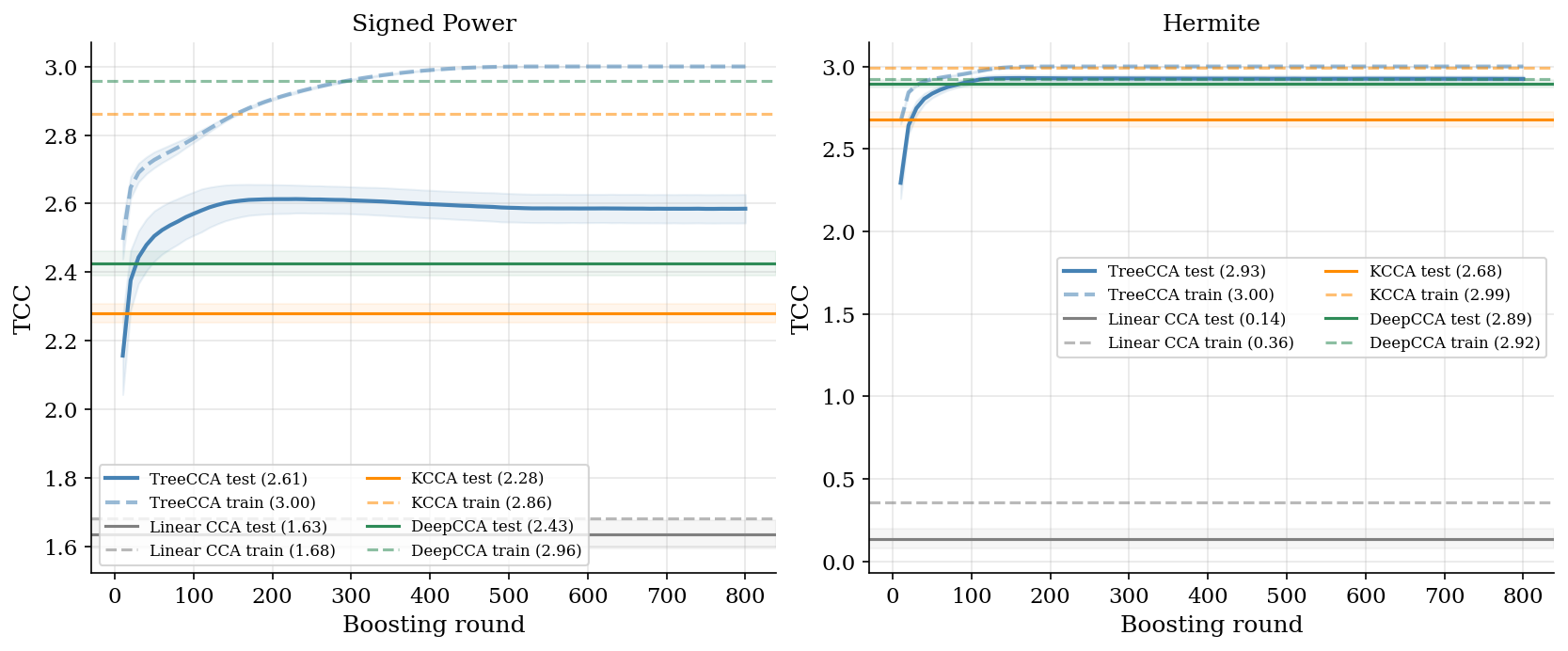}
  \caption{Test TCC (solid) and train TCC (dashed) vs.\ boosting round for TreeCCA,
           with Linear CCA, KCCA, and DeepCCA (EY loss) shown as horizontal baselines.
           \textbf{Left}: Signed Power; TreeCCA reaches 2.62 ($+56\%$ over linear, $+8\%$ over DCCA).
           \textbf{Right}: Hermite; linear CCA is near-zero (0.13); TreeCCA reaches
           $22\times$ improvement (99\% of oracle), marginally above DCCA.}
  \label{fig:benchmarks}
\end{figure}

\section{TreeCCA as Alternating Regression}
\label{app:alt_reg}

\paragraph{Why alternating least squares converges to CCA.}
Classical CCA can be solved by alternating regression~\citep{golub1995canonical,xu2019als}.
Given fixed embeddings $Z_2 = X_2 W_2$, the optimal linear encoder for view 1 is the
solution to the least-squares problem
\[
  \min_{W_1}\; \|X_1 W_1 - Z_2\|_F^2,\quad
  \text{giving}\quad W_1 \leftarrow (X_1^\top X_1)^{-1} X_1^\top Z_2,
\]
and symmetrically for $W_2$.  Convergence to the CCA solution follows because the
alternating update is a block coordinate descent on the CCA objective: at each step,
the within-view Gram matrix $X_v^\top X_v$ is implicitly used to normalise the
projection, enforcing the CCA orthonormality constraint $W_v^\top \Sigma_{vv} W_v = I_K$
in the limit.  Each subproblem has a unique closed-form solution, and the objective
decreases monotonically, so convergence is guaranteed~\citep{xu2019als}.

\paragraph{Why this approach fails for gradient-boosted trees.}
Replacing the linear learner with a GBT encoder in the alternating LS framework breaks
in two ways.

\emph{Scale and component collapse.}
The LS objective $\|Z_1 - Z_2\|_F^2$ drives $Z_1 \to Z_2$, collapsing both embeddings
toward the same point.  GBT leaf regularisation ($\lambda\|\mathrm{leaf}\|^2$) then
shrinks both toward zero.  Even if collapse is prevented for one component, without
orthogonality constraints all $K$ dimensions converge to the same dominant direction.

\emph{The normalisation bottleneck.}
In the linear case, normalisation is implicit: the closed-form solution automatically
accounts for the within-view covariance.  For a GBT there is no closed-form solution;
normalisation would require refitting all trees from scratch after each step to enforce
$Z_v^\top Z_v = (N-1) I_K$, making the method computationally prohibitive.

\paragraph{How the EY loss resolves both problems.}
The EY gradient
\[
  G_{1,i} \propto -Z_{2c,i} + \V\, Z_{1c,i}
\]
resolves both failure modes structurally.  The cross-view term $-Z_{2c,i}$ is the
alternating LS signal; the within-view term $+\V Z_{1c,i}$ penalises covariance
expansion, replacing the implicit normalisation constraint with an explicit gradient
penalty.  No separate normalisation step is required, and GBT leaf regularisation then
provides step-size control without collapsing the embeddings.

\paragraph{Gauss-Seidel vs.\ Jacobi.}
Classical alternating regression is strict Gauss-Seidel: update $W_1$, re-compute
$Z_1$, then compute $G_2$ from the fresh $Z_1$.  TreeCCA uses Jacobi: both gradients
$G_1, G_2$ are computed from the same $(Z_1, Z_2)$ before either booster is updated.
In practice the two schedules are indistinguishable ---
peak TCC differs by $\leq 0.004$ on both benchmarks (Appendix~\ref{app:gs}) --- and
Jacobi is preferred for its simplicity.

\section{Experimental Details}
\label{app:details}

\paragraph{Hardware and software.}
All experiments run on a single MacBook Pro (Apple M3).
Python 3.12, XGBoost 2.1, LightGBM 4.3, scikit-learn 1.5, NumPy 1.26.

\paragraph{GBT hyperparameters (all experiments unless stated).}
TreeCCA (\texttt{train\_treecca}, default): $K$ scalar XGBoost Boosters per view.
\texttt{tree\_method='hist'}, \texttt{base\_score=0.0},
\texttt{learning\_rate=0.1}, \texttt{max\_depth=5}, \texttt{subsample=0.8},
\texttt{colsample\_bytree=0.8}, \texttt{min\_child\_weight=5}.
HAR experiment: \texttt{learning\_rate=0.05}, \texttt{max\_depth=3} (reduced depth
prevents overfitting on $p=36$ features).
TreeCCA-Joint (\texttt{train\_treecca\_joint}): same parameters plus
\texttt{multi\_strategy='multi\_output\_tree'}, \texttt{base\_score=[0.0]*K}.
Unit Hessians are used throughout (\S\ref{sec:method}).

\paragraph{DCCA configuration.}
Two hidden layers: $[128, 64]$ units, ReLU activations, BatchNorm.
Adam optimiser, learning rate $10^{-3}$, 1000 epochs.
Trained via PyTorch Lightning on CPU.
Split MNIST: $[1024, 1024]$ units, no BatchNorm.

\paragraph{Reproducibility.}
All synthetic datasets use \texttt{numpy.random.default\_rng} with seed 42 (single-seed
experiments) or seeds $\{42,0,1,2,3\}$ (multi-seed experiments).

\section{GBT Design Comparison: TreeCCA vs.\ TreeCCA-Joint}
\label{app:design}

A central implementation question is how to structure the $K$-dimensional encoder per
view.  We compare three configurations on both benchmarks (Figure~\ref{fig:design}):

\begin{itemize}
  \item \textbf{TreeCCA} (\texttt{train\_treecca}): $K$ independent scalar XGBoost
        boosters per view.  Recommended default.
  \item \textbf{TreeCCA-LightGBM} (\texttt{train\_treecca\_lgbm}): same independent
        architecture using LightGBM, isolating the library effect.
  \item \textbf{TreeCCA-Joint} (\texttt{train\_treecca\_joint}): one
        \texttt{multi\_output\_tree} XGBoost booster per view, all $K$ gradient columns
        coupled in a single split-finding pass.
\end{itemize}

\paragraph{Fair comparison: trees per view.}
TreeCCA adds $K$ trees per view per round while TreeCCA-Joint adds 1.  Running
TreeCCA-Joint for $K\times$ as many rounds equalises the cumulative tree count.
On this axis (Figure~\ref{fig:design}, centre), TreeCCA-Joint converges faster per tree
on both benchmarks: joint split-finding aggregates $K$ gradient signals and is more
information-efficient per tree.

\paragraph{Wall-clock time.}
Despite needing fewer trees, TreeCCA-Joint is slower in wall-clock time
(Figure~\ref{fig:design}, right): each joint tree maintains $K$-wide histograms,
and this overhead exceeds the per-tree saving.

\paragraph{Library effect and recommendation.}
TreeCCA and TreeCCA-LightGBM nearly coincide, confirming library choice has negligible
impact when the architecture is fixed.
Use \textbf{TreeCCA} when training speed matters (the common case);
use \textbf{TreeCCA-Joint} when model size or inference latency is a priority.

\begin{figure}[H]
  \centering
  \includegraphics[width=\linewidth]{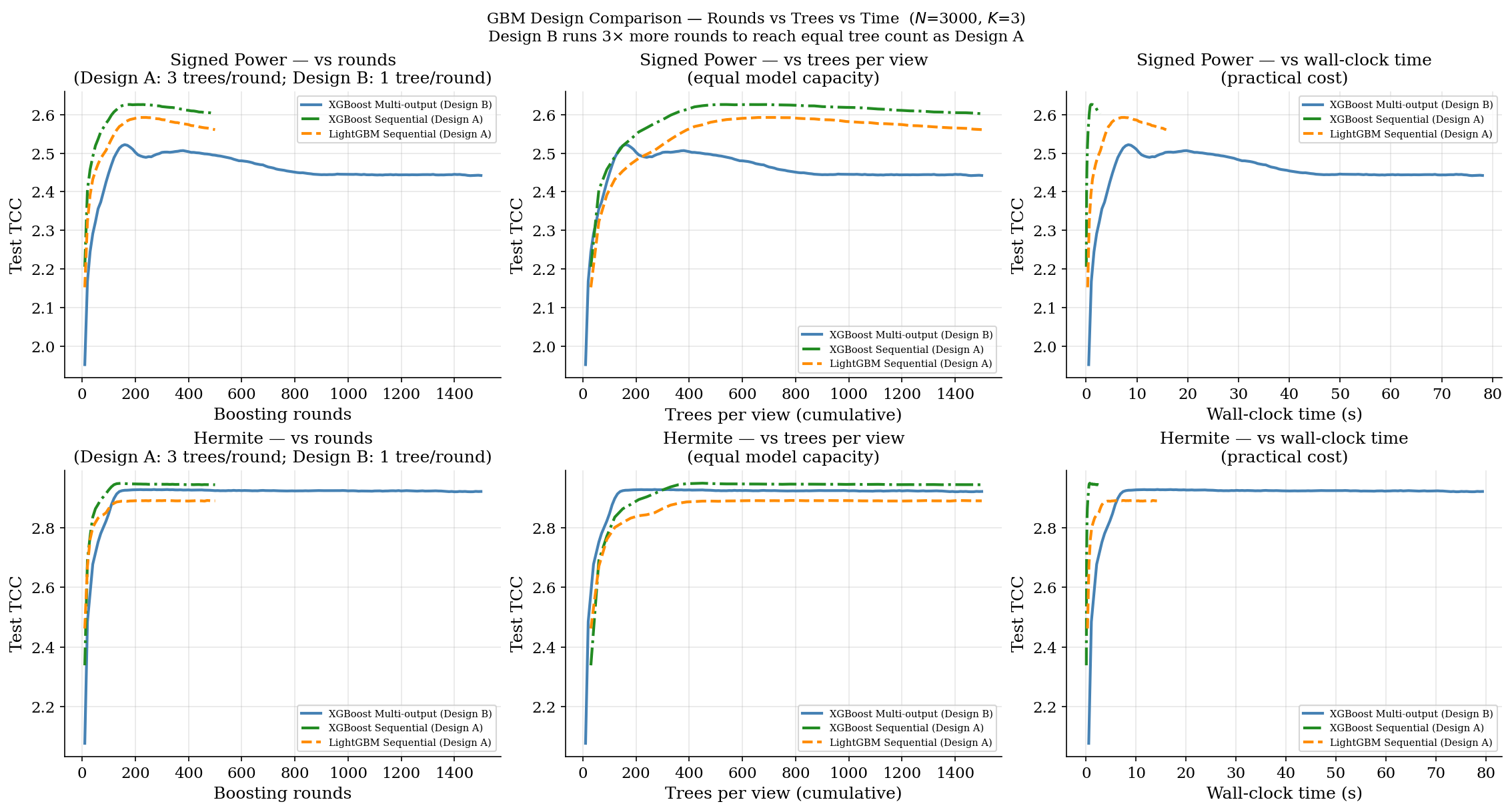}
  \caption{TreeCCA vs.\ TreeCCA-Joint vs.\ TreeCCA-LightGBM on signed-power (top) and
           Hermite (bottom).
           \textbf{Left}: vs rounds (TreeCCA adds $K$ trees per round).
           \textbf{Centre}: vs trees per view (fair) --- TreeCCA-Joint converges faster
           per tree.
           \textbf{Right}: vs wall-clock --- TreeCCA is faster in practice.}
  \label{fig:design}
\end{figure}

\section{Hyperparameter Sensitivity}
\label{app:hp}

Figure~\ref{fig:hp} shows that TreeCCA is broadly robust to hyperparameter choice: across
\texttt{max\_depth} $\in \{3,4,5,6\}$ and \texttt{learning\_rate} $\in \{0.10,0.20\}$
(500 rounds, seed 42), peak test TCC varies by less than $0.07$.
The only configuration that underperforms is \texttt{lr}$=0.05$, which has not converged
at 500 rounds; any faster learning rate works well.

\begin{figure}[H]
  \centering
  \includegraphics[width=0.55\linewidth]{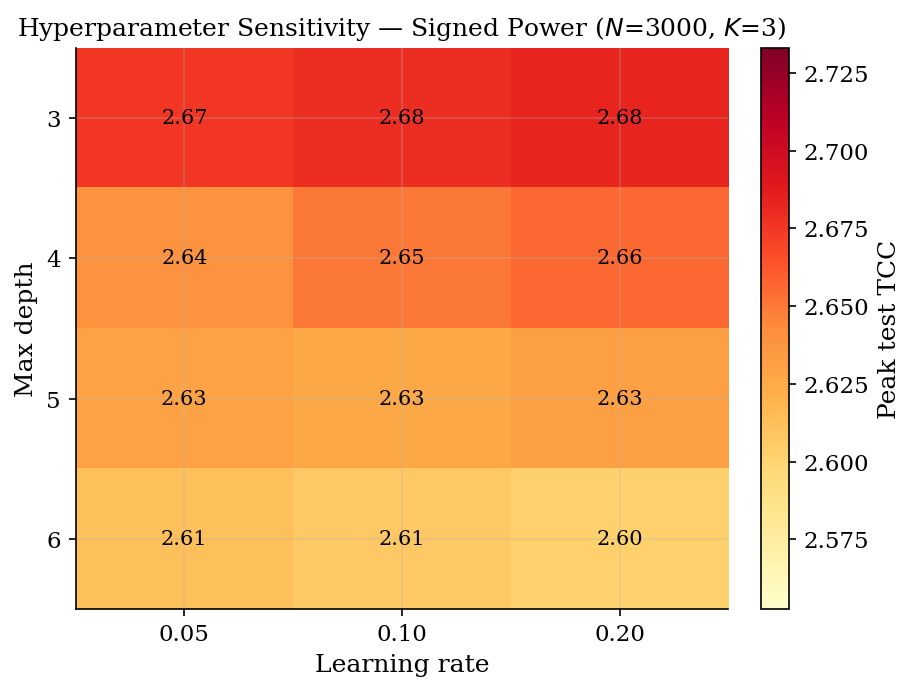}
  \caption{Hyperparameter sensitivity heatmap (signed-power benchmark).  Performance
           is robust at \texttt{lr}$\geq$0.10 across all depths tested.}
  \label{fig:hp}
\end{figure}

\section{Scalability}
\label{app:scalability}

Each TreeCCA round fits $2K$ scalar trees, giving total complexity $O(R \cdot K \cdot N \cdot p)$.
Figure~\ref{fig:scale} plots test TCC against wall-clock time for sweeps over $N$
(left, fixed $p=K=3$) and $p$ (right, fixed $N=5{,}000$).

The key message: the TCC ceiling is the same at every scale.  Curves shift right as
$N$ or $p$ grows, but all reach the same final performance within 200 rounds.
Per-round timing remains well under one second even at $N=100{,}000$ or $p=3{,}000$
on a single CPU core.

\begin{figure}[H]
  \centering
  \includegraphics[width=\linewidth]{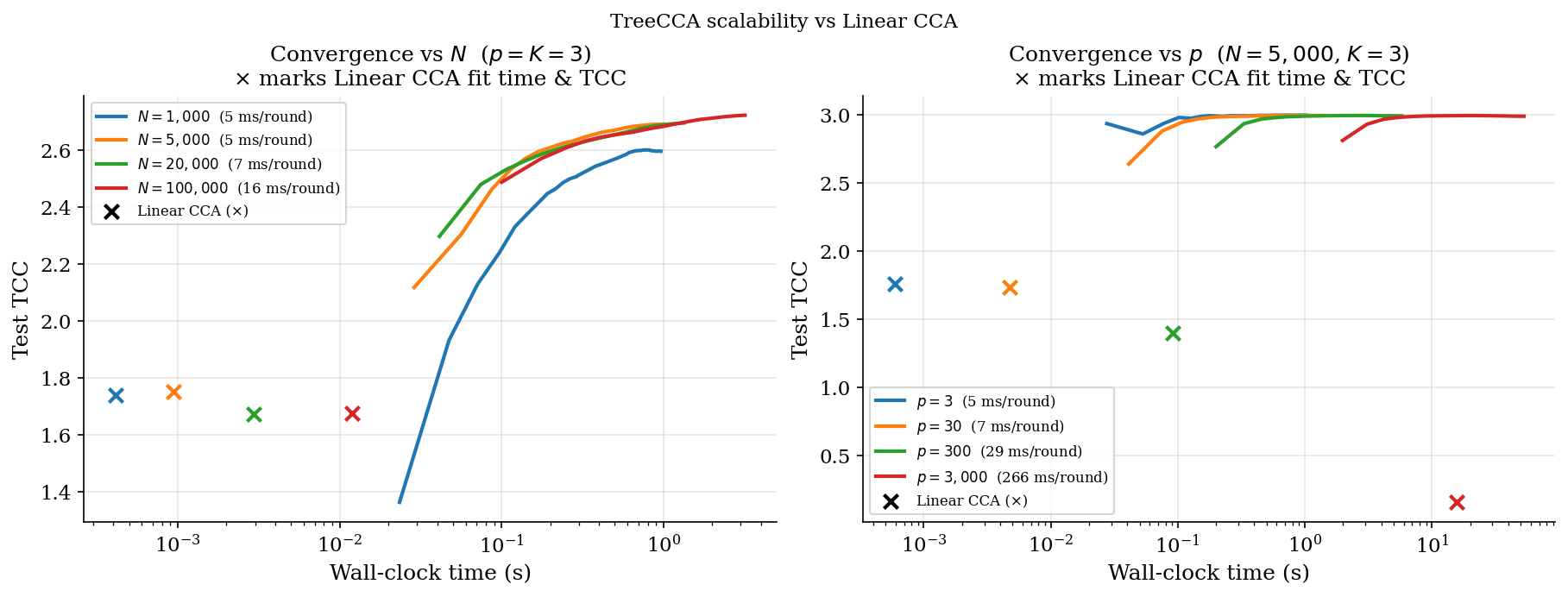}
  \caption{\textbf{TreeCCA convergence.}  Test TCC vs.\ wall-clock time on the
           signed-power benchmark.
           \textbf{Left}: varying $N$ (fixed $p=K=3$) --- consistent with $O(N)$
           per-round cost.
           \textbf{Right}: varying $p$ (fixed $N=5{,}000$) --- consistent with $O(Np)$.}
  \label{fig:scale}
\end{figure}

\section{Initialisation Ablation}
\label{app:init}

Figure~\ref{fig:init} compares five initialisation strategies on the Hermite benchmark,
averaged over 5 random seeds.  PCA unscaled (the default) reaches the highest peak test
TCC.  All non-zero strategies eventually escape the saddle; zero initialisation is
permanently stuck because $G=0$ exactly when $Z=0$.

\begin{figure}[H]
  \centering
  \includegraphics[width=\linewidth]{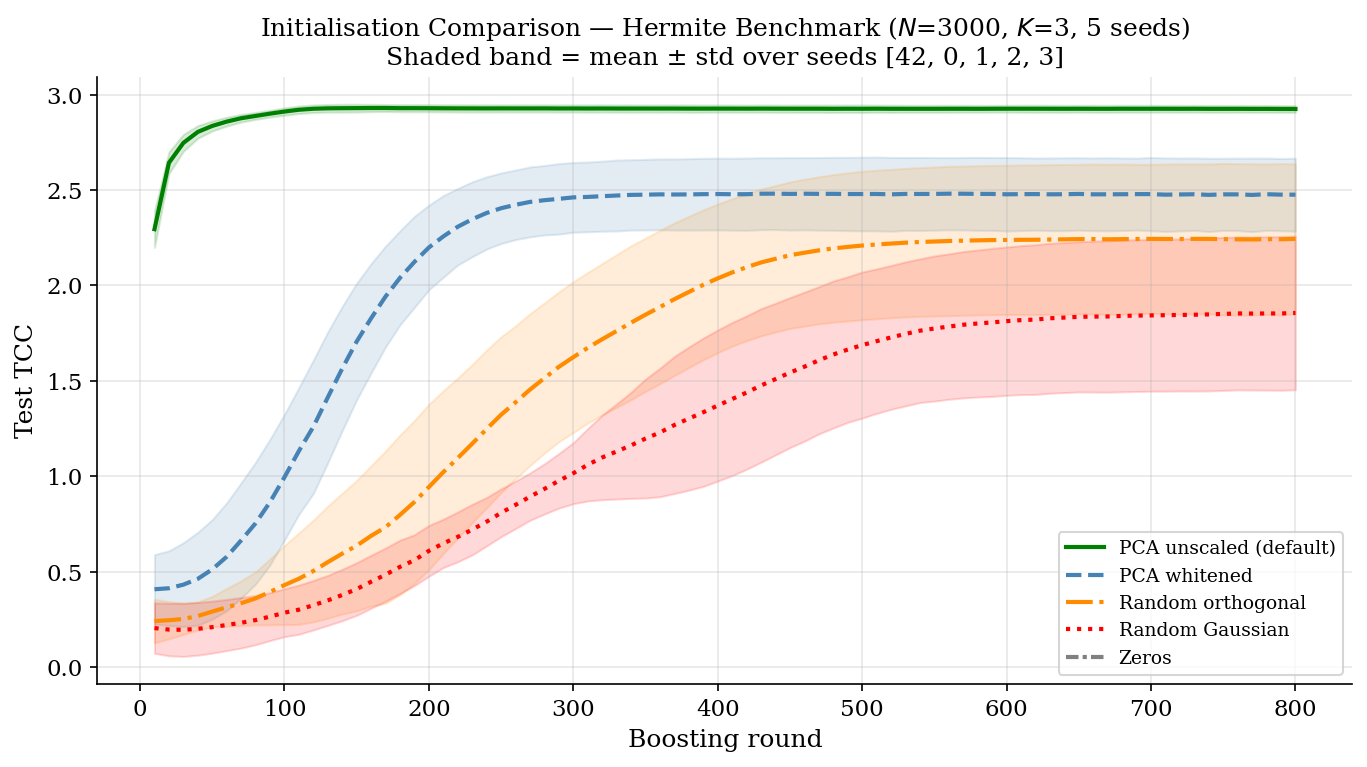}
  \caption{Initialisation comparison on the Hermite benchmark (mean $\pm$ std, 5 seeds).
           PCA unscaled (default) converges fastest and reaches the highest peak TCC.
           All non-zero strategies escape the saddle; zero init is permanently stuck.}
  \label{fig:init}
\end{figure}

\section{Jacobi vs.\ Gauss-Seidel Updates}
\label{app:gs}

As discussed in Appendix~\ref{app:alt_reg}, TreeCCA uses Jacobi updates (both gradients
computed before either booster is updated), whereas Gauss-Seidel would re-predict $Z_1$
after updating $\mathcal{B}_1$ before computing $G_2$.  Figure~\ref{fig:gs} confirms the
two schedules are indistinguishable empirically.

\begin{figure}[H]
  \centering
  \includegraphics[width=0.49\linewidth]{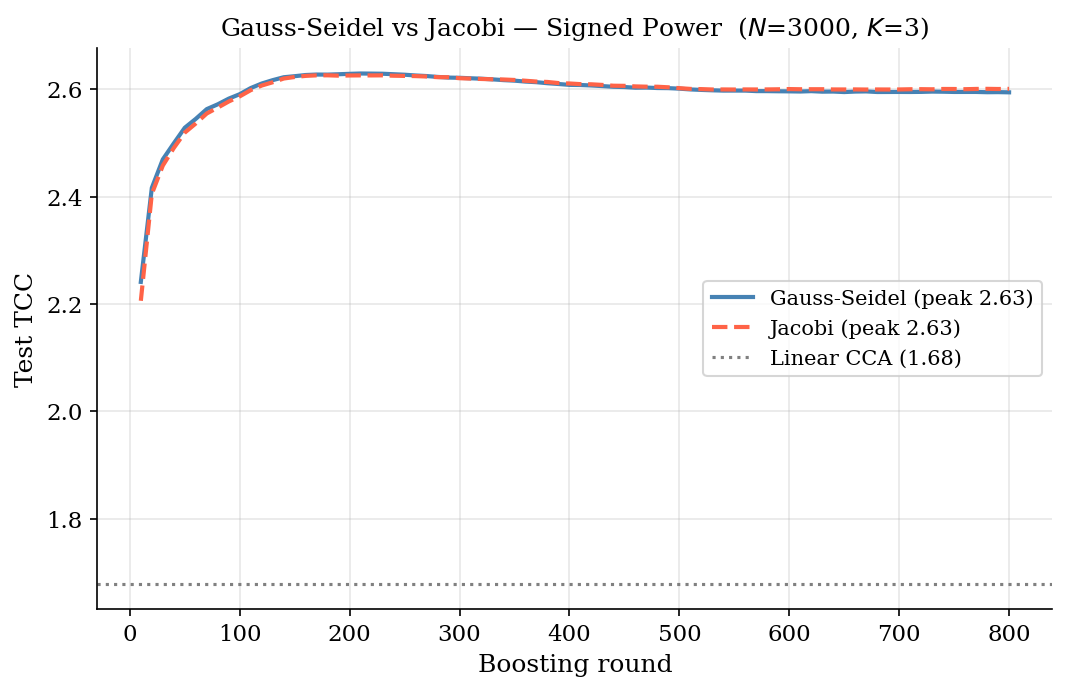}
  \includegraphics[width=0.49\linewidth]{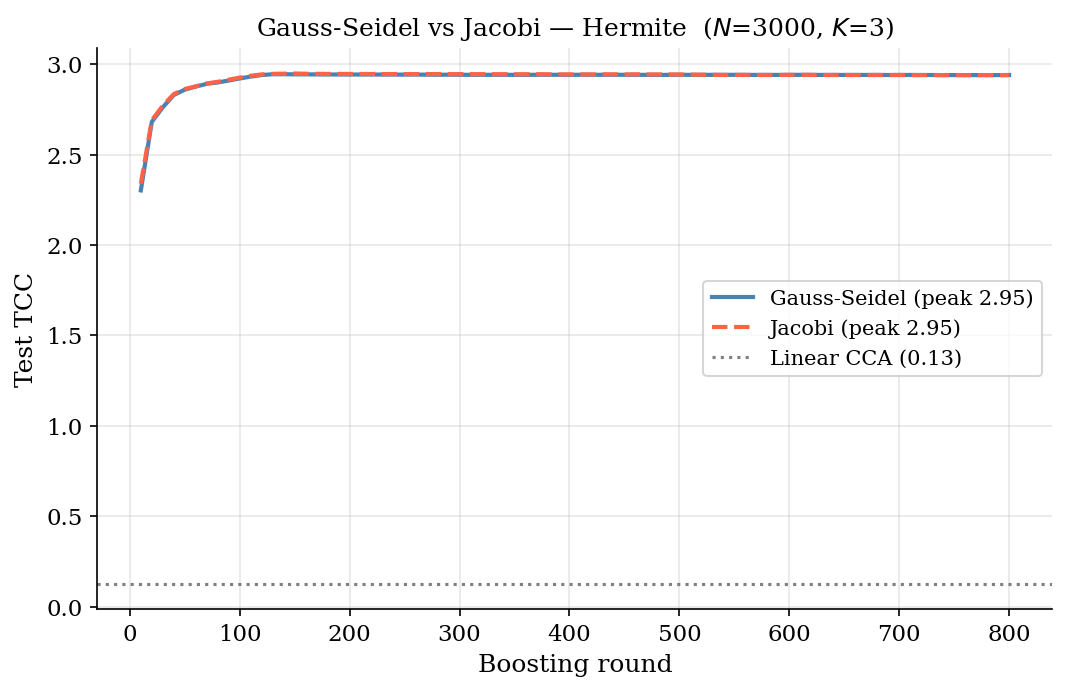}
  \caption{Gauss-Seidel vs.\ Jacobi on signed-power (left) and Hermite (right).
           Both schedules achieve identical peak TCC.}
  \label{fig:gs}
\end{figure}

\section{GBT Custom Objectives: Gradient, Hessian, and Leaf Values}
\label{app:hessian}

This section is a self-contained derivation of the two implementation choices in
TreeCCA --- gradient normalisation and unit Hessians --- starting from how XGBoost's
custom objective API works.

\paragraph{How XGBoost uses a custom objective.}
At each boosting round, the user supplies two length-$N$ arrays: a per-sample gradient
$g_i$ and a per-sample Hessian $h_i$.  XGBoost grows one tree by finding the split
partition that maximises gain, then assigns each leaf $\mathcal{I}_\ell$ the optimal weight
\[
  w^*_\ell = -\frac{\displaystyle\sum_{i \in \mathcal{I}_\ell} g_i}
                   {\displaystyle\sum_{i \in \mathcal{I}_\ell} h_i + \lambda},
\]
where $\lambda$ is the L2 leaf regularisation (default 1).  The prediction update is
$\Delta\hat{y}_i = \eta \cdot w^*_{\ell(i)}$, where $\eta$ is the learning rate.

The \textbf{denominator} controls step size.  For standard least-squares regression,
$g_i = \hat{y}_i - y_i$ and $h_i = 1$ (Hessian of $\tfrac{1}{2}(y_i-\hat{y}_i)^2$),
giving denominator $n_\ell + \lambda$ and leaf value equal to the \emph{average} residual
in the leaf, shrunk by $\lambda$.  This averaging is what makes gradient boosting stable:
each leaf summarises rather than accumulates its samples.

\paragraph{The raw EY gradient.}
The EY gradient with respect to $Z_{1,i,k}$ is (Eq.~\ref{eq:grad}):
\[
  g_i^{\mathrm{EY}} = \frac{4}{N-1}\bigl(-Z_{2c,i,k} + [\mathbf{V}Z_{1c}]_{i,k}\bigr).
\]
At PCA-unscaled initialisation, embedding columns have unit $\ell_2$ norm, so the
bracket term has magnitude $O(1/\sqrt{N})$ per element (dominated by $-Z_{2c,i,k}$).
The factor $\tfrac{4}{N-1}$ makes the full gradient $O(1/N\sqrt{N})$.

\paragraph{Choice 1 (correctness): Unit Hessians.}
We pass $h_i = 1$ for all samples.  The leaf formula becomes
\[
  w^*_\ell = -\frac{\sum_{i \in \mathcal{I}_\ell} g_i}{n_\ell + \lambda},
\]
i.e.\ the \emph{average} gradient in the leaf.  This is identical to the standard
gradient-boosted regression update (Hessian of $\tfrac{1}{2}(y-\hat{y})^2$ is also 1).

The alternative --- the true EY Hessian $h_i = \tfrac{4}{N-1}V_{kk}$ --- is $O(1/N^2)$
at PCA initialisation ($V_{kk}^{(0)} = 2/(N-1)$):
\[
  h_i^{(0)} = \frac{8}{(N-1)^2} \implies
  \text{denominator} = 94 \cdot \frac{8}{2999^2} + 1 \approx 0.00008 + 1 \approx 1.
\]
The leaf size term is negligible; $\lambda$ dominates.  The leaf value equals the
gradient \emph{sum} over $n_\ell$ samples rather than the average --- roughly
$n_\ell \approx 94\times$ too large --- and training diverges immediately.  This problem
does not improve as training proceeds: even at $V_{kk} = 1$ the true Hessian is
$4/(N-1) \approx 0.0013$, giving denominator $94 \times 0.0013 + 1 \approx 1.12$,
still $\lambda$-dominated.  Unit Hessians are therefore required for correctness.

\paragraph{Choice 2 (practical contribution): Gradient normalisation.}
With unit Hessians and the raw gradient (including the $\tfrac{4}{N-1}$ prefactor),
the leaf value is
\[
  w^*_\ell = -\frac{\sum g_i^{\mathrm{EY}}}{n_\ell + \lambda}
           = -\frac{4}{N-1} \cdot \mathrm{mean}(g^{\mathrm{bracket}}_\ell).
\]
For $N=3{,}000$ this is $\approx 0.0013 \times O(1/\sqrt{N}) \approx 0.000024$ per
round --- converging to the same solution, just $\sqrt{N} \approx 55\times$ more
slowly than necessary.  More importantly, the update scale depends on $N$: the same
learning rate gives very different effective step sizes on different dataset sizes.

We therefore drop the $\tfrac{4}{N-1}$ factor and normalise the bracket jointly across
both views:
\[
  \tilde{G}_v = \frac{-Z_{\bar{v}c} + \mathbf{V}Z_{vc}}
                     {\max\!\bigl(\sigma(G_1),\, \sigma(G_2),\, \epsilon\bigr)}
                \cdot 0.1.
\]
The result has standard deviation $\approx 0.1$ at every round regardless of $N$ or the
current state of $\mathbf{V}$.  Leaf updates are then $O(\eta \cdot 0.1 / n_\ell)$ at
any dataset size: the learning rate is uniformly interpretable, and no per-$N$
re-tuning is required.  This is a practical contribution, not a mathematical necessity.

\paragraph{Summary.}
\textbf{Unit Hessians} are required for correctness: the true EY Hessian causes the
$\lambda$ term to dominate the denominator, giving gradient sums rather than averages
and causing divergence.  \textbf{Gradient normalisation} is a practical contribution:
training without it converges to the same solution but $O(\sqrt{N})$ more slowly, and
the effective learning rate varies with $N$.  Together they implement the standard
regression-tree update rule --- gradient-averaged-over-leaf --- at consistent speed
across all dataset sizes.

Figure~\ref{fig:hessian} shows the diagonal $V_{kk}$ rising from $\approx 0$ toward
$\approx 1$ during training.  This reflects the growth of within-view embedding
variance as the EY objective is minimised and is shown as an optimisation diagnostic;
it does not enter the gradient or Hessian arrays passed to XGBoost.

\begin{figure}[H]
  \centering
  \includegraphics[width=0.85\linewidth]{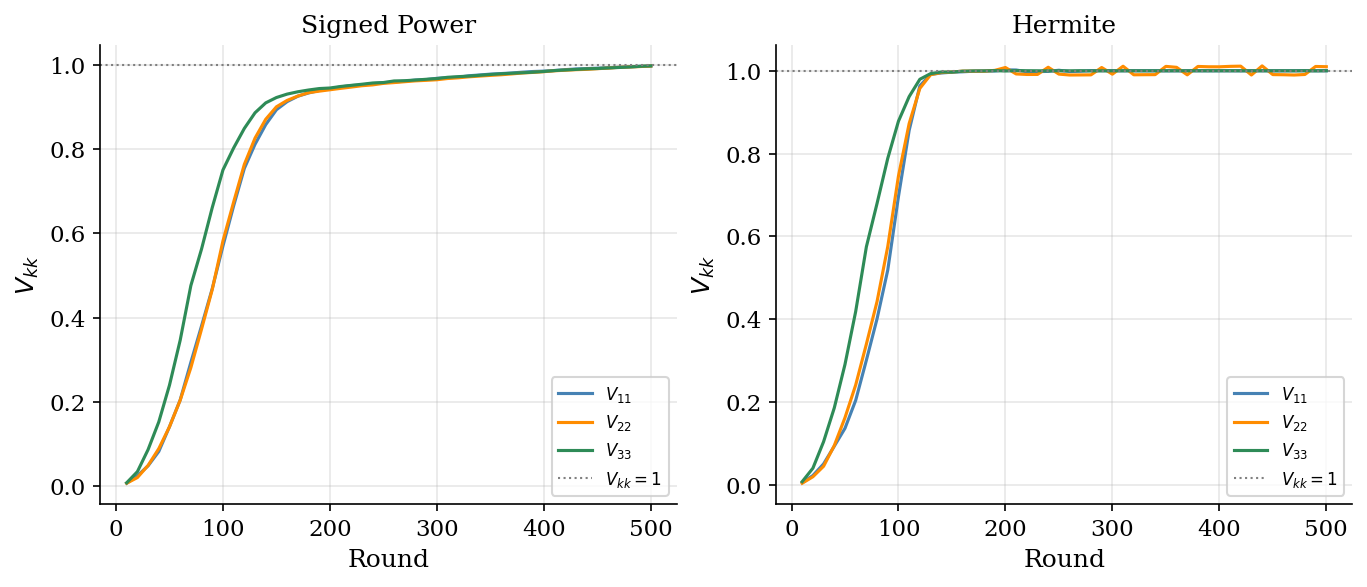}
  \caption{Per-component $V_{kk}$ over boosting rounds (seed 42).  $V_{kk}$ rises
           from $\approx 0$ at PCA-unscaled initialisation toward $\approx 1$ by round
           $\sim$200, reflecting growth of within-view embedding variance as training
           converges.  This is an optimisation diagnostic; it does not affect the
           gradient or Hessian arrays passed to XGBoost.}
  \label{fig:hessian}
\end{figure}

\section{Multi-Seed Stability Curves}
\label{app:multiseed}

\begin{figure}[H]
  \centering
  \includegraphics[width=0.6\linewidth]{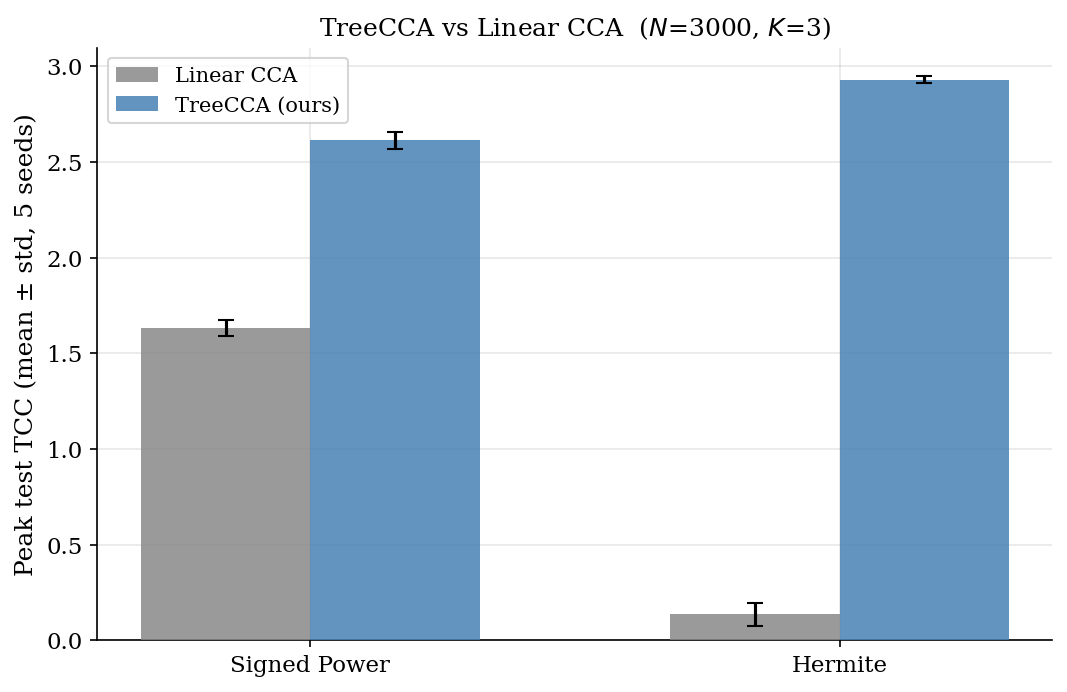}
  \caption{Peak test TCC (mean $\pm$ std, 5 seeds $\{42,0,1,2,3\}$) on both benchmarks.
           TreeCCA (blue) vs.\ linear CCA (grey).  Numerical summary in
           Table~\ref{tab:nonlinear}.}
  \label{fig:multiseed}
\end{figure}

\section{TreeMCCA: Experimental Results}
\label{app:multiview}

\begin{algorithm}[h]
\caption{TreeMCCA (multi-view extension of Algorithm~\ref{alg:treecca})}
\label{alg:treemcca}
\begin{algorithmic}[1]
\REQUIRE Views $X_1,\ldots,X_M \in \R^{N\times p_m}$;\;
         embedding dim $K$;\; rounds $T$;\; learning rate $\eta$
\ENSURE $K$ scalar boosters per view; embed via
        $\hat{Z}_m = [b_{m,1}(X_m),\ldots,b_{m,K}(X_m)]$
\STATE Initialise $MK$ scalar boosters $b_{m,k}$, base margin $\leftarrow$ $k$-th PCA direction of $X_m$
\STATE $Z_m \leftarrow \mathrm{PCA}_K(X_m)$ \hfill\COMMENT{initial embeddings, $m\in\{1,\ldots,M\}$}
\FOR{$t = 1, \ldots, T$}
  \STATE $\tilde{G}_m \leftarrow {-\sum_{j \neq m} Z_{jc}}
         + Z_{mc}\!\bigl[(M{-}1)V_{mm} + \textstyle\sum_{j \neq m} V_{jj}\bigr]$,
         normalised \hfill\COMMENT{all-pairs EY gradient from current embeddings, all $m$ simultaneously (Jacobi)}
  \STATE $b_{m,k} \mathrel{+}= \mathrm{tree}(X_m,\;\tilde{G}_{m,k},\;H{=}1)$
         \hfill\COMMENT{fit and add one new tree per booster; all $m$, $k$}
  \STATE $Z_{m,k} \leftarrow Z_{m,k} + \eta\;\Delta b_{m,k}^{(t)}(X_m)$
         \hfill\COMMENT{update cached embedding with new tree only; all $m$, $k$}
\ENDFOR
\RETURN $\{b_{m,k}\}_{m=1,k=1}^{M,K}$
\end{algorithmic}
\end{algorithm}

The only change from Algorithm~\ref{alg:treecca} is the gradient in line~5: the
all-pairs EY objective $\sum_{i<j}\mathcal{L}_{\mathrm{EY}}(Z_i,Z_j)$ yields a
gradient for view $m$ that aggregates attraction from all other views.
At $M=2$ this reduces exactly to Algorithm~\ref{alg:treecca}.

We construct a four-view dataset ($N=3000$, $K=3$) using four distinct nonlinear
transforms of the same latents: $v_0 = \sgn(z)|z|^{1/3}$, $v_1 = \sgn(z)|z|^3$,
$v_2 = z^2-1$ (Hermite $H_2$, even), $v_3 = z^3-3z$ (Hermite $H_3$, odd).
Hermite-orthogonal pairs ($v_0$--$v_2$, $v_2$--$v_3$) have near-zero linear
cross-correlation, making linear CCA progressively less effective as $M$ grows.

Figure~\ref{fig:multiview} reports average pairwise test TCC across all $\binom{M}{2}$
pairs for $M \in \{2,3,4\}$.  Linear CCA average collapses below $1.0$ at $M=4$;
TreeMCCA degrades gracefully.  Crucially, only the gradient computation changes across
$M$ --- no encoder architecture modification is needed.

\begin{figure}[H]
  \centering
  \includegraphics[width=0.75\linewidth]{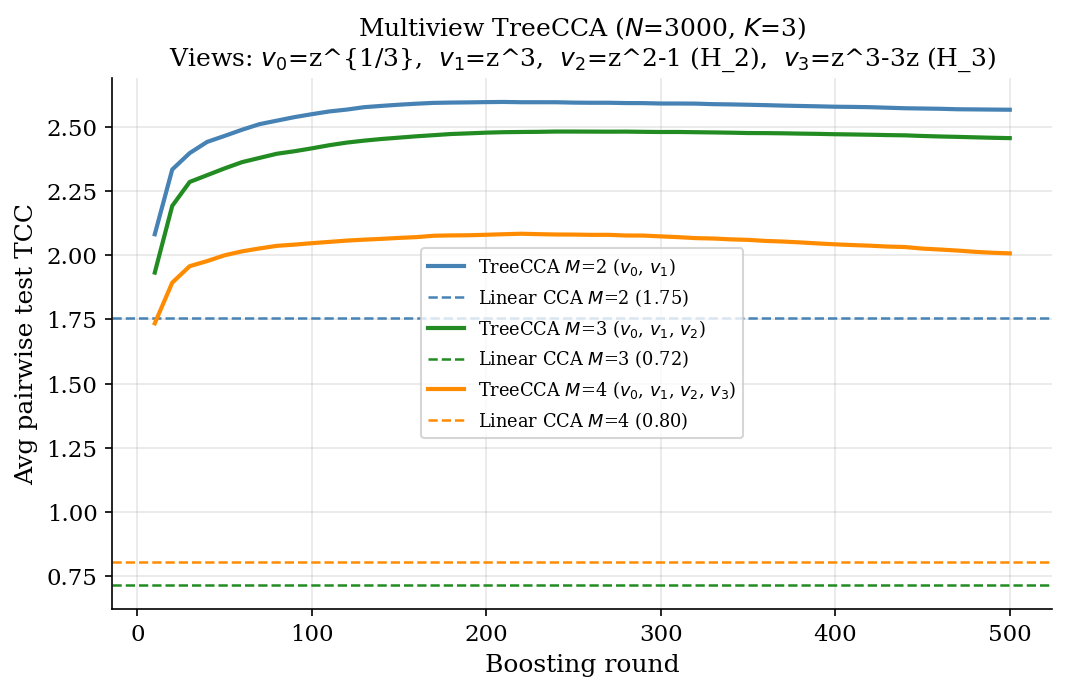}
  \caption{Average pairwise test TCC vs.\ boosting round for $M=2,3,4$ views.
           Dashed lines show linear CCA average pairwise baseline.
           As $M$ grows, linear CCA collapses (dashed) while TreeCCA maintains
           substantially higher per-pair TCC.}
  \label{fig:multiview}
\end{figure}

\section{TreeCCA-SSL: Full Experimental Details}
\label{app:ssl}

\begin{algorithm}[h]
\caption{TreeCCA-SSL (siamese single-encoder variant)}
\label{alg:treecca_ssl}
\begin{algorithmic}[1]
\REQUIRE Unlabelled data $X \in \R^{N\times p}$;\; augmentation $\mathcal{A}$;\;
         embedding dim $K$;\; rounds $T$;\; learning rate $\eta$
\ENSURE $K$ scalar boosters $\{b_k\}$; embed via $\hat{Z} = [b_1(X),\ldots,b_K(X)]$
\STATE Initialise $K$ scalar boosters $b_k$ with random orthogonal base margins
\STATE $Z \leftarrow$ random orthogonal init of shape $N\times K$
\FOR{$t = 1, \ldots, T$}
  \STATE Draw two independent augmented views: $V \sim \mathcal{A}(X)$,\; $V' \sim \mathcal{A}(X)$
  \STATE $Z_{k}  \leftarrow b_k^{(t)}(V)$,\quad
         $Z'_{k} \leftarrow b_k^{(t)}(V')$
         \hfill\COMMENT{full forward pass each round; augmented views change}
  \STATE $\tilde{G},\tilde{G}' \leftarrow \tfrac{4}{N-1}(-Z'_c + \mathbf{V}Z_c)$,\;
         $\tfrac{4}{N-1}(-Z_c + \mathbf{V}Z'_c)$, normalised
         \hfill\COMMENT{EY gradient, both directions}
  \STATE $b_k \mathrel{+}= \mathrm{tree}\!\bigl(V,\;\tfrac{1}{2}(\tilde{G}_k+\tilde{G}'_k),\;H{=}1\bigr)$
         \hfill\COMMENT{shared booster updated on symmetrised gradient}
\ENDFOR
\RETURN $\{b_k\}_{k=1}^K$ \hfill\COMMENT{apply to original $X$, not augmented views}
\end{algorithmic}
\end{algorithm}

The key differences from Algorithm~\ref{alg:treecca}: a \emph{single} set of $K$
boosters is shared across both views; augmented views are resampled each round so
the full forward pass is required (no incremental cache); and the gradient is
symmetrised over both augmentation directions.  At inference, the boosters are
applied to the original unaugmented features.

\subsection*{Data Generating Process}
\label{app:ssl_dgp}

\paragraph{Latent factors and class labels.}
$K_{\mathrm{lat}}=6$ orthonormal factors $Z_k \sim \mathcal{N}(0,1)$ are drawn and
orthogonalised.  Class $y = \arg\max_k |Z_k|$.

\paragraph{Sign-interaction features.}
Each of $D=60$ observed features encodes a product interaction:
\begin{equation}
  X_j = Z_{a_j} \cdot \sgn(Z_{b_j}), \quad a_j \neq b_j,
  \label{eq:dgp}
\end{equation}
normalised to unit standard deviation.  The magnitude $|X_j|=|Z_{a_j}|$ is
class-informative, but the sign depends on the unobserved reference factor $Z_{b_j}$.
Key properties:
\begin{itemize}
  \item \textbf{PCA-fatal}: $\mathrm{Cov}(X) = I_D$ --- all linear directions are
        equivalent.  No linear combination of $X$ correlates with any $|Z_k|$.
  \item \textbf{Tree-amenable}: $|X_j| = |Z_{a_j}|$ --- trees recover $|Z_{a_j}| > c$
        via the pair of splits $X_j > c$ and $X_j < -c$.
\end{itemize}

\paragraph{Augmentation.}
At each round, two independent views are drawn:
\begin{equation}
  V_j = X_j \cdot \mathrm{flip}_j + \sigma\varepsilon_j,\quad
  V_j' = X_j \cdot \mathrm{flip}_j' + \sigma\varepsilon_j',
  \label{eq:aug}
\end{equation}
where $\mathrm{flip} \overset{\mathrm{iid}}{\sim} \mathrm{Uniform}(\{-1,+1\})$ and
$\varepsilon, \varepsilon' \sim \mathcal{N}(0,1)$, all mutually independent, $\sigma = 0.2$.

\subsection*{Theoretical Analysis}

\begin{proposition}[Affine encoders cannot extract cross-view signal]
\label{prop:linear_fail}
Let $f(V) = VW + \mathbf{1}b^\top$ be any affine encoder (including PCA).  Under
augmentation Eq.~\eqref{eq:aug}, the expected cross-view covariance of any representation
is zero:
\[
  \E_{V,V'}\bigl[\mathrm{Cov}(f(V),\, f(V'))\bigr] = W^\top \E[V^\top V'] W = 0.
\]
No cross-view objective can learn from these augmented views using an affine encoder.
PCA also fails because $\mathrm{Cov}(X)=I_D$ makes all linear directions equivalent.
\end{proposition}

\begin{proof}
For any features $j,l$:
\[
\E[V_j V_l'] = \underbrace{\E[X_j X_l]\E[\mathrm{flip}_j]\E[\mathrm{flip}_l']}_{=\,0}
+ \underbrace{\sigma^2\E[\varepsilon_j]\E[\varepsilon_l']}_{=\,0}
+ \text{cross terms} = 0,
\]
since $\E[\mathrm{flip}]=0$ and $\varepsilon, \varepsilon'$ are independent with mean
zero.  Hence $\E[V^\top V']=0$.
\end{proof}

\subsection*{Results}

TreeCCA-SSL is trained with $K_{\mathrm{emb}}=5$ embedding dimensions, $N=4000$,
1500 rounds, random orthogonal initialisation.  Downstream performance: linear probe
accuracy on the 6-way classification task (6 classes, random chance $= 0.167$).

The key intuition is that the class-informative signal lives in the \emph{magnitudes}
$|X_j| = |Z_{a_j}|$, not in the signs.  A linear encoder cannot recover this because
$\mathrm{Cov}(X)=I_D$: all linear directions are equivalent and the cross-view signal
cancels in expectation (Proposition~\ref{prop:linear_fail}).  Trees can recover it by
learning paired splits $\{X_j > c\} \cup \{X_j < -c\}$, which extract $|X_j|$
implicitly.  The oracles in Table~\ref{tab:ssl} bracket the achievable performance at
each level of privileged information:

\begin{itemize}
  \item \textbf{PCA$(|X|)$}: a cheating oracle that knows to take absolute values before
        applying PCA.  It cannot be matched by an unsupervised method that sees only $X$.
  \item \textbf{Oracle $|X|$}: a linear probe trained directly on the clean (noise-free)
        magnitude features $|X_j|$.  Stronger than PCA$(|X|)$ because it skips the PCA
        compression step.
  \item \textbf{Oracle $|Z|$}: a linear probe on the true latent factors $Z_k$
        themselves --- the theoretical ceiling, unreachable in practice.
\end{itemize}

\begin{figure}[H]
  \centering
  \includegraphics[width=\linewidth]{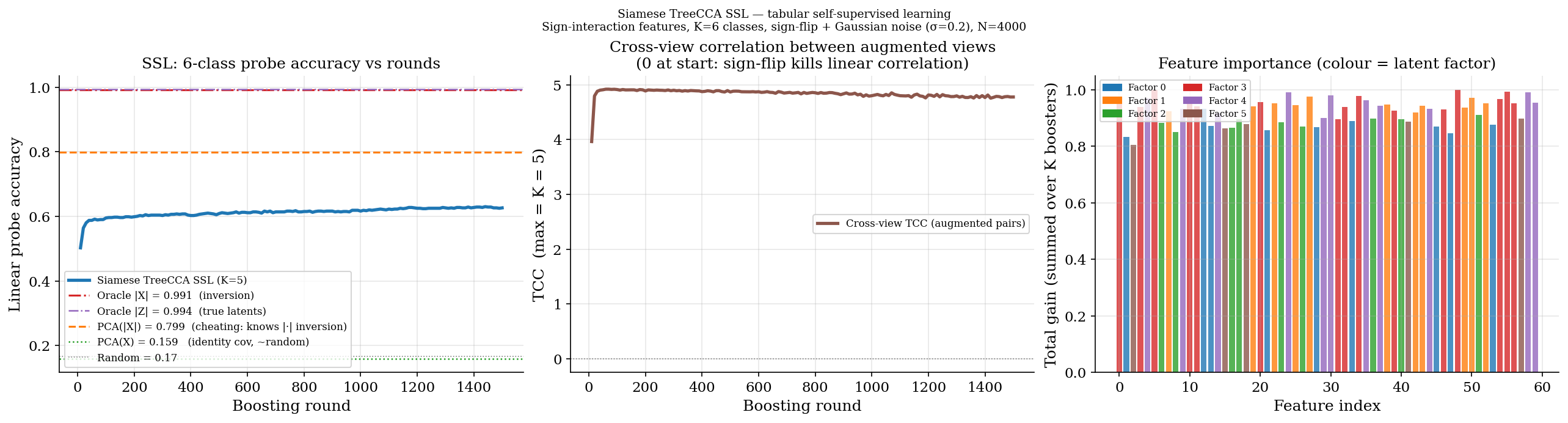}
  \caption{TreeCCA-SSL on the sign-interaction benchmark ($N=4000$, $K=6$ classes).
           \textbf{Left}: linear probe accuracy vs.\ boosting round.  TreeCCA-SSL
           (blue, 0.63) rises nearly $4\times$ above PCA$(X)$ (0.159, $\approx$ random).
           \textbf{Centre}: EY objective vs.\ round.
           \textbf{Right}: feature importance by latent factor.}
  \label{fig:ssl}
\end{figure}

\begin{table}[h]
  \centering
  \caption{Linear probe accuracy on sign-interaction SSL benchmark
           ($K_{\mathrm{lat}}=6$, $N=4000$, 5-dim embedding, seed 42).}
  \label{tab:ssl}
  \begin{tabular}{lcc}
    \toprule
    Method & Accuracy & Notes \\
    \midrule
    Random chance             & 0.167 & $1/K_{\mathrm{lat}}$, lower bound \\
    PCA$(X)$ + linear probe   & 0.159 & Provably at random chance (Prop.~\ref{prop:linear_fail}) \\
    \midrule
    \textbf{TreeCCA-SSL}      & \textbf{0.630} & No labels, no knowledge of sign structure \\
    \midrule
    PCA$(|X|)$ + linear probe & 0.799 & Oracle: knows to take $|\cdot|$ before PCA \\
    Oracle $|X|$              & 0.991 & Oracle: linear probe on clean magnitude features \\
    Oracle $|Z|$ true latents & 0.994 & Theoretical ceiling \\
    \bottomrule
  \end{tabular}
\end{table}

Proposition~\ref{prop:linear_fail} proves PCA is stuck at random chance; the empirical
result ($0.159 \approx 1/6$) confirms this.  TreeCCA-SSL reaches $0.630$
($3.8\times$ above random) with no labels and no knowledge of the sign structure ---
placing it well above the linear baseline and meaningfully below the oracle that cheats
by knowing $|\cdot|$ is the right preprocessing.  The remaining gap to $0.799$ reflects
augmentation noise ($\sigma=0.2$): the flip augmentation corrupts the sign information
that a cheating oracle exploits, so the achievable ceiling for an honest method is
lower than PCA$(|X|)$.

\section{Multi-View Embedding Visualisations}
\label{app:tsne}

Figure~\ref{fig:tsne} shows t-SNE projections~\citep{vandermaaten2008tsne} of the
test-set embeddings for each real-world dataset (seed 42, concatenated $M$-view
embeddings coloured by class label).  Linear CCA embeddings (top row) are often
poorly separated; TreeMCCA embeddings (bottom row) show substantially tighter
class clusters, consistent with the probe accuracy gains in Table~\ref{tab:realworld_mat}.

\begin{figure}[H]
  \centering
  \includegraphics[width=\linewidth]{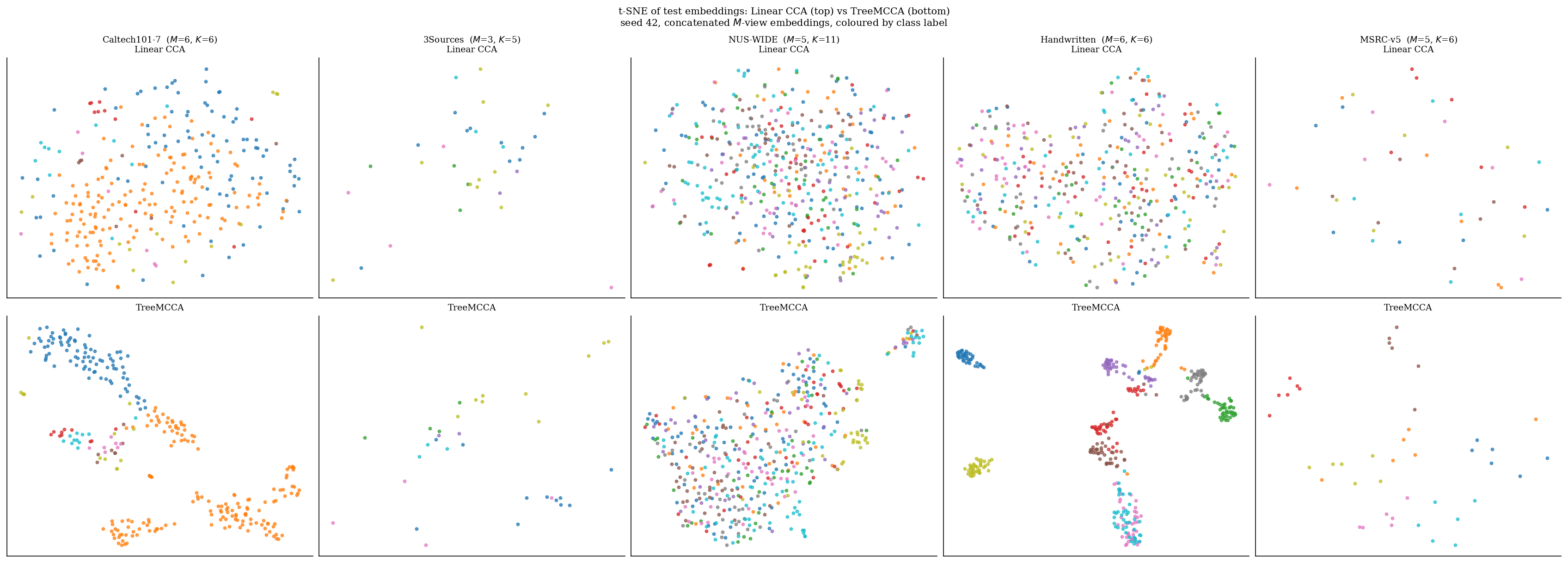}
  \caption{t-SNE of test embeddings.  \textbf{Top row}: Linear CCA
           (pairwise average projection).  \textbf{Bottom row}: TreeMCCA
           (all-views joint).  Each column is one dataset; points coloured by
           class label.  Seed 42, $K=C{-}1$ components.}
  \label{fig:tsne}
\end{figure}

\end{document}